\documentclass[article]{elsarticle} 

\usepackage{hyperref}

\journal{Journal of \LaTeX\ Templates}

\usepackage{amsmath,amssymb,epsfig}
\usepackage{graphicx}
\usepackage{makeidx} 
\usepackage{mathptmx}       
\usepackage{helvet}         
\usepackage{courier}        
\usepackage{type1cm} 
\usepackage{multicol}
\usepackage[bottom]{footmisc}
\usepackage[utf8]{inputenc}
\usepackage[english]{babel}
\usepackage{amsthm}
\usepackage{epstopdf,overpic}
\usepackage[utf8]{inputenc}
\usepackage[table]{xcolor}
\usepackage{microtype}

\newtheorem{theorem}{Theorem}

\newtheorem{lemma}{Lemma}
\newtheorem{Definition}{Definition}
\newtheorem{remark}{Remark}

\begin{document}

\begin{frontmatter}

\title{Linear Pentapods with a Simple Singularity Variety}

\author{Arvin Rasoulzadeh \fnref{Arvin}}
\author{Georg Nawratil \fnref{Georg}}
\address{Center for Geometry and Computational Design, Vienna University of Technology, Austria}
\fntext[Arvin]{E-mail address: rasoulzadeh@geometrie.tuwien.ac.at.}
\fntext[Georg]{E-mail address: nawratil@geometrie.tuwien.ac.at.}



\begin{abstract}
There exists a bijection between the configuration space of a linear pentapod and all points $(u,v,w,p_x,p_y,p_z)\in\mathbb{R}^{6}$ 
located on the singular quadric $\Gamma: u^2+v^2+w^2=1$, where $(u,v,w)$ determines the orientation of the linear platform and $(p_x,p_y,p_z)$ 
its position. 
Then the set of all singular robot configurations is obtained by intersecting $\Gamma$ with a cubic hypersurface $\Sigma$ in $\mathbb{R}^{6}$, 
which is only quadratic in the orientation variables and position variables, respectively.  
This article investigates the restrictions to be imposed on the design of this mechanism in order to obtain a reduction in degree. 
In detail we study the cases where $\Sigma$ is (1) linear in position variables, (2) linear in orientation variables and (3) quadratic in total. 
The resulting designs of linear pentapods have the advantage of considerably simplified computation of singularity-free spheres in the 
configuration space. 
Finally we propose three kinematically redundant designs of linear pentapods with a
simple singularity surface. 
\end{abstract}

  
\end{frontmatter}


\section{Introduction} 
\label{introduction}
A \emph{linear pentapod} (cf.\ Fig.\ \ref{fig:1} ) is defined as a five degree-of-freedom (DOF) \emph{line-body component} of a Gough-Stewart platform consisting of a linear motion platform $\ell$ with five identical spherical-prismatic-spherical ($\mathrm{S}\underline{\mathrm{P}}\mathrm{S}$) legs, where the prismatic joints are active and the rest are passive \cite{kong2001generation}. The pose of $\ell$ is uniquely determined by a position vector $\mathbf{p}\in \mathbb{R}^{3}$ and an orientation given by a unit-vector $\mathbf{i}\in \mathbb{R}^{3}$. 
The coordinate vector $\mathbf{m}_{j}$ of the platform anchor point $m_j$ of the $j$-th leg is defined by the equation 
$\mathbf{m}_{j}=\mathbf{p}+r_{j}\mathbf{i}$ and the base anchor points $M_j$ of the $j$-th leg has coordinates $\mathbf{M}_{j}=(x_{j}, y_{j}, z_{j})^{T}$ for $j=1,\ldots ,5$.

\begin{figure}[h!] 
\begin{center}   
  \begin{overpic}[height=65mm]{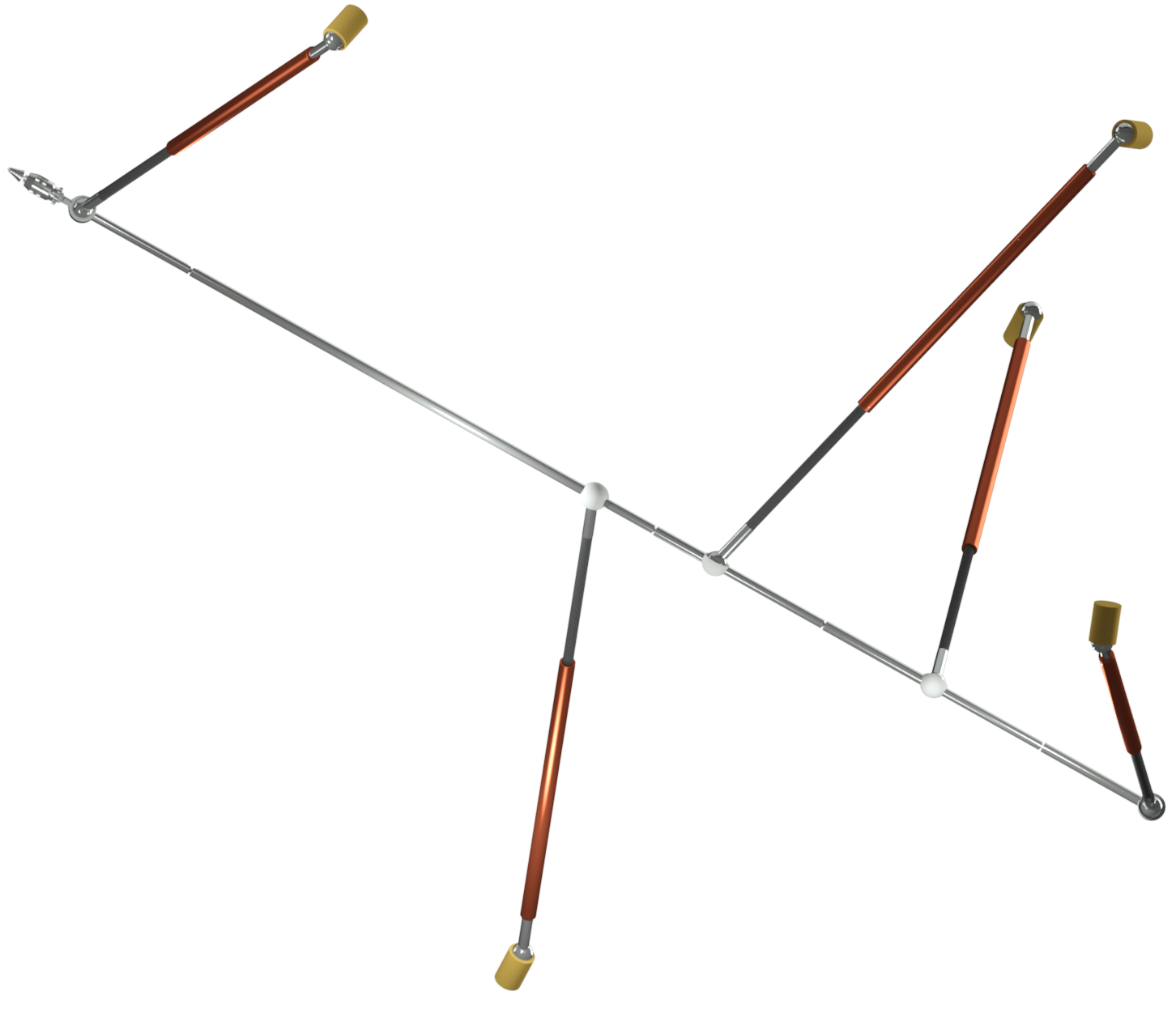}
	\put(34.5,54){$\ell$}
  \end{overpic} 
	\caption{Sketch of a linear pentapod.}
	\label{fig:1}
\end{center}
\end{figure}

It turns out that this kind of manipulator is an interesting alternative to serial robots handling axis-symmetric tools. Some fundamental industrial tasks such as 5-axis milling, laser engraving and water jet cutting are counted as its applications in industry \cite{borras2010singularity, weck2002parallel}.

Singularity analysis plays an important role in motion planning of parallel manipulators. Special configurations referred to as \emph{kinematic singularities} have always been central in mechanism theory and robotics. Beside being an intellectually appealing topic, the study of kinematic singularities provides an insight of major practical and theoretical importance for the design, control, and application of robot manipulators.

In such singularities, the \emph{kinetostatic} properties of a mechanism undergo sudden and dramatic changes. This motivates the enormous practical value of a careful study and thorough understanding of the phenomenon for the design and use of manipulators.

\subsection{Review}
\label{review}

The singularity analysis of linear pentapods has undergone an acceptable level of investigations over the past few years. In the 
following we give an overview of the obtained results:

From the line-geometric point of view (cf. \cite{merlet1989singular}) a linear pentapod is in a singular configuration if and only if 
the five carrier lines of the legs belong to a linear line congruence \cite{pottmann2009computational}; i.e. the Pl\"ucker coordinates of 
these lines are linearly dependent. From this latter characterization the following algebraic one can be obtained (cf.\ 
\cite{rasoulzadeh2018rational}):

There exists a bijection between the configuration space of a linear pentapod and all points 
$(u,v,w,p_x,p_y,p_z)\in\mathbb{R}^{6}$ located on the singular quadric $\Gamma: u^2+v^2+w^2=1$, where $(u,v,w)$ 
determines the orientation of the linear platform $\ell$ and $(p_x,p_y,p_z)$ its position. 
Then the set of all singular robot configurations is obtained as the intersection of $\Gamma$ with a cubic hypersurface $\Sigma$ of $\mathbb{R}^{6}$, 
which can be written as $\Sigma:\, det(\mathbf{S})=0$ with 
\begin{eqnarray}
\mathbf{S}=\left( \begin {array}{ccccccc} 1&u&v&w&p_{x}&p_{y}&p_{z}
\\ 0&p_{x}&p_{y}&p_{z}&0&0&0
\\ 0&0&0&0&u&v&w\\ r_{2}&x_{2}
&y_{2}&z_{2}&r_{2}x_{2}&r_{2}y_{2}&r_{2}z_{2}
\\ r_{3}&x_{3}&y_{3}&z_{3}&r_{3}x_{3}&r_
{3}y_{3}&r_{3}z_{3}\\ r_{4}&x_{4}&y_{4
}&z_{4}&r_{4}x_{4}&r_{4}y_{4}&r_{4}z_{4}
	\\ r_{5}&x_{5}&y_{5}&z_{5}&r_{5}x_{5}&r_
{5}y_{5}&r_{5}z_{5}\end {array} \right) 
\label{BorrasMatrix}
\end{eqnarray}
 (according to \cite{borras2010singularity})
under the assumption that $x_1=y_1=z_1=r_1=0$. Note that this assumption can always be made 
without loss of generality as the fixed/moving frame can always be chosen in a way that the first base/platform anchor point 
is its origin. 
Moreover, a rational parametrization of the singularity loci $\Gamma \cap \Sigma$ was given by the authors in \cite{rasoulzadeh2018rational}. 

A singular configuration can also be characterized as a multiple solution of the direct kinematics problem. In this context it should be mentioned that 
the forward kinematics of a linear pentapod was solved for the first time in \cite{zhang1992forward} under the assumption of a planar base, and 
in \cite{nawratil2015self} for the general case. If the direct kinematics problem has a continuous solution, then the linear pentapod has a 
so-called self-motion. All designs of linear pentapods possessing such motions are listed in \cite{nawratil2015self}. A more detailed 
study of the corresponding self-motions is performed in \cite{nawratil2018line}. Moreover the last two cited papers also contain extensive 
literature reviews on this topic.

A further well-studied field within the singularity analysis of linear pentapods are designs, which are  singular in any configuration. 
These so-called architecture singular designs are completely classified in \cite[Section 1.3]{nawratil2015self}, where also all relevant references 
in this context are cited. 

Finally it should be noted, that
Borr$\grave{\textrm{a}}$s and Thomas have studied how to move the leg attachments in the base and the platform of $5$-$\mathrm{S}\underline{\mathrm{P}}\mathrm{S}$ linear pentapod without altering the robot's singularity locus (for a planar base see \cite{borras2011singularity} and for a non-planar one see \cite{borras2010singularity}).

\subsection{Motivation and outline}
\label{motivation}

Using a parallel manipulator with a simple singularity variety (with respect to the position variables) 
was first proposed by Karger \cite{karger2006stewart} for the case of Stewart-Gough platforms\footnote{For Stewart-Gough platforms the singularity loci is in general cubic in the position variables.}. 
This work was furthered in  \cite{nawratil2010stewart} and \cite{nawratil2010}, where the necessary conditions for the design of Stewart-Gough platforms with linear or quadratic singularity surface with respect to positioning variables are determined.

It can easily be seen that the equation of the cubic hypersurface $\Sigma$ is only quadratic in position as well as in orientation variables. 
Therefore the intention here is to find necessary conditions for the linear pentapods such that $det(\mathbf{S})=0$ is: 
\begin{itemize}
\item[$\bullet$] linear in position variables (cf.\ Section \ref{sec:2}), 
\item[$\bullet$] linear in orientation variables (cf.\ Section \ref{sec:3}),
\item[$\bullet$] quadratic in total (cf.\ Section \ref{sec:4}).
\end{itemize}     

Clearly, due to the degree reduction it becomes easier to obtain closed form information 
about singular positions. But the main motivation for our research is the computational simplification 
of singularity-free zones, for which the state of art is as follows:

In \cite{rasoulzadeh2018rational} it is proven that for a generic linear pentapod, 
the computation of the maximal singularity-free zone in the position/orientation workspace (with respect to the 
Euclidean/spherical metric) leads over to the solution of a polynomial of degree 6 and 8, respectively. 
The corresponding closest singular configurations in the position/orientation workspace are illustrated in Fig.\ \ref{fig:general}-left.

\begin{figure}[t!] 
\begin{center}   
  \begin{overpic}[width=59mm]{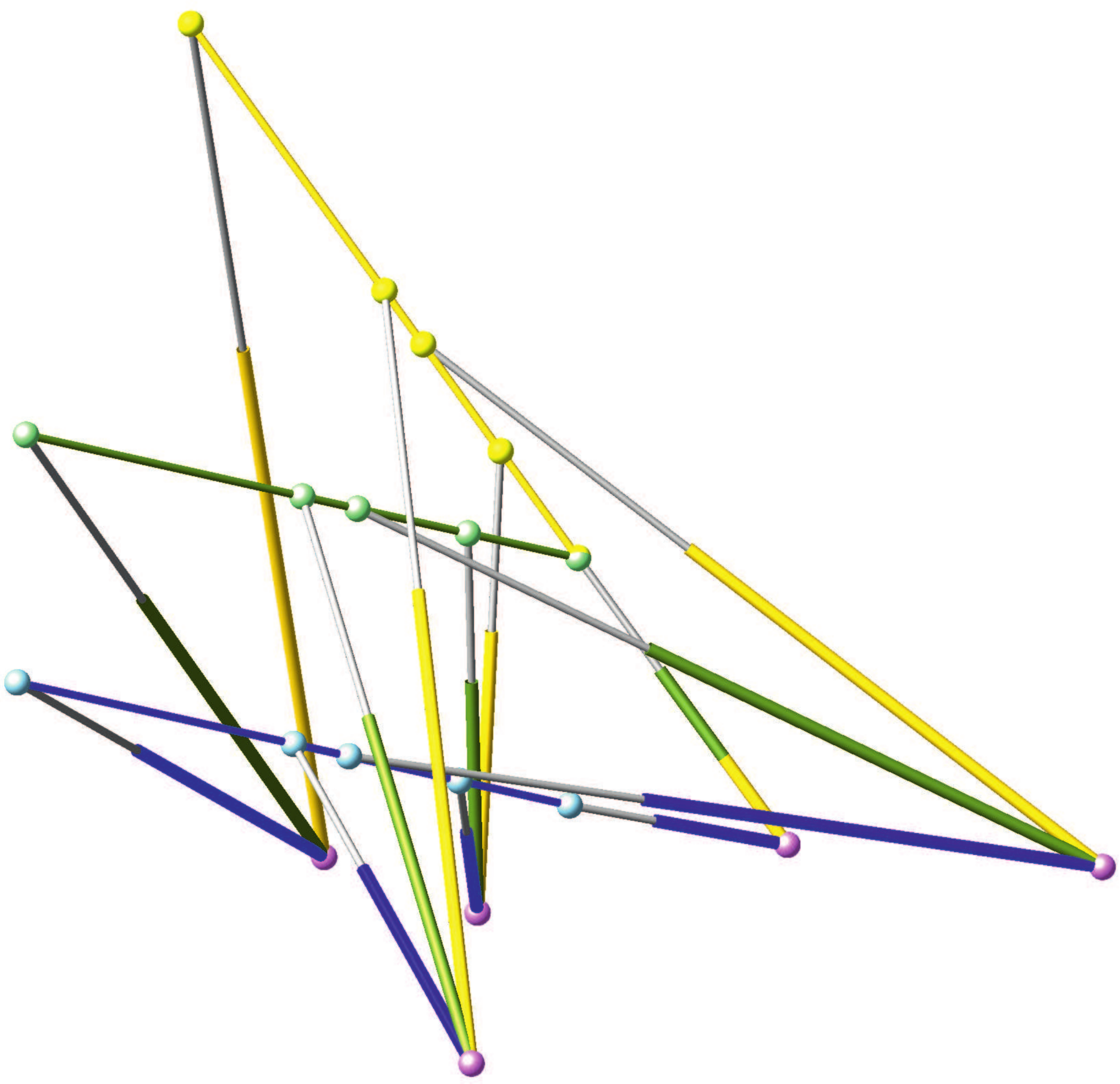}
	\begin{small}
	\put(28,83){$\mathfrak{P}$}
	\put(10,58){$\mathfrak{G}$}
	\put(7,36.5){$\mathfrak{O}$}
	\end{small}     
  \end{overpic} 
	\hfill
	 \begin{overpic}[width=60mm]{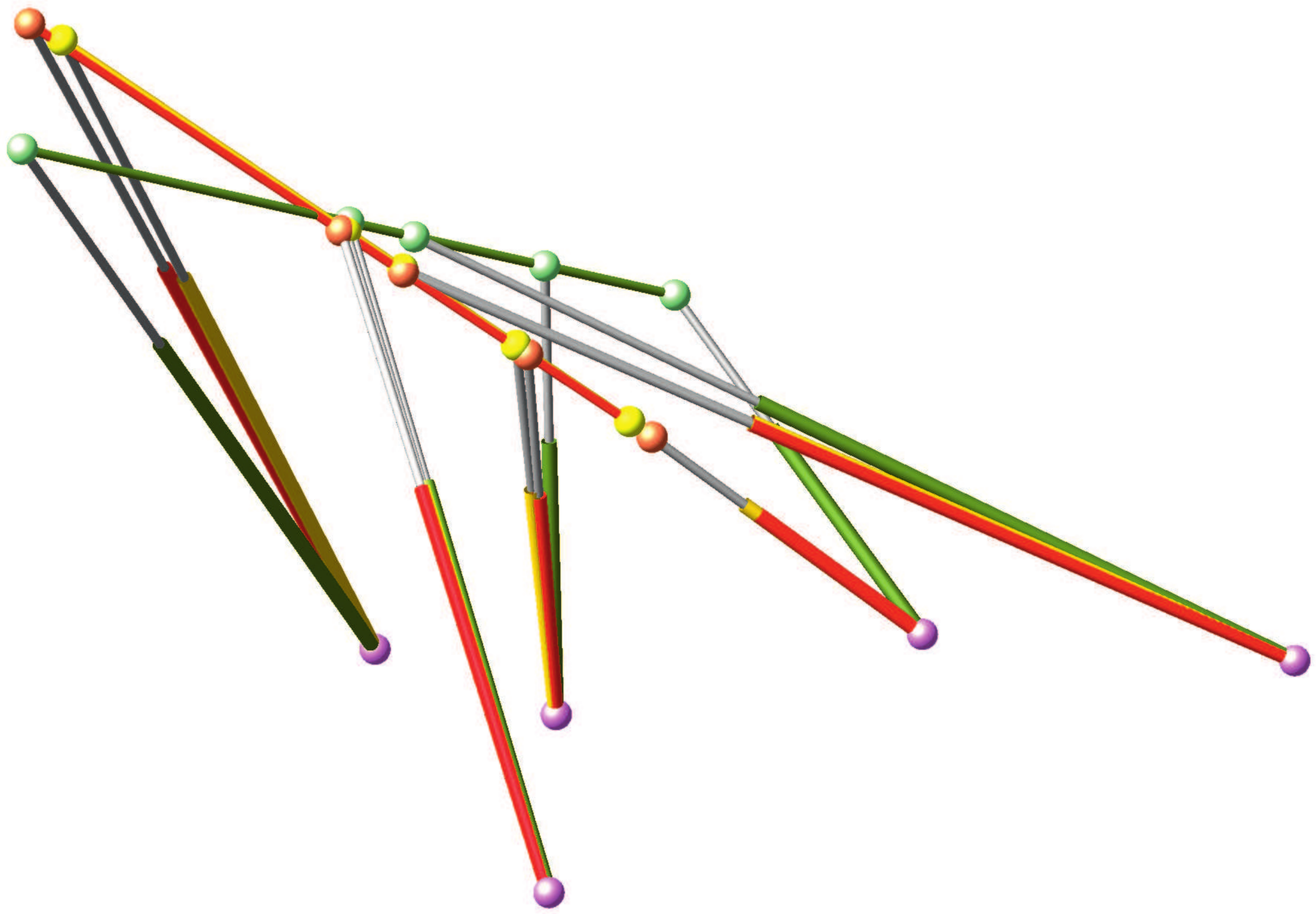}
	\begin{small}
	\put(6,66.5){$\mathfrak{N}$}
	\put(-5,66){$\mathfrak{M}$}
	\put(-4,55){$\mathfrak{G}$}
	\end{small}     
  \end{overpic} 
	\caption{
	Given is the pose $\mathfrak{G}$ (green) of the linear pentapod. 
	Left: The closest singular configurations in the position/orientation workspace 
	are given by the pose $\mathfrak{P}$ (yellow) and $\mathfrak{O}$ (blue), respectively. 
	Right: $\mathfrak{M}$ (red) is the closest singular pose under Euclidean motions of $\ell$ and
	$\mathfrak{N}$ (yellow) is the closest singularity under equiform motions  of $\ell$.}
	\label{fig:general}
\end{center}
\end{figure}
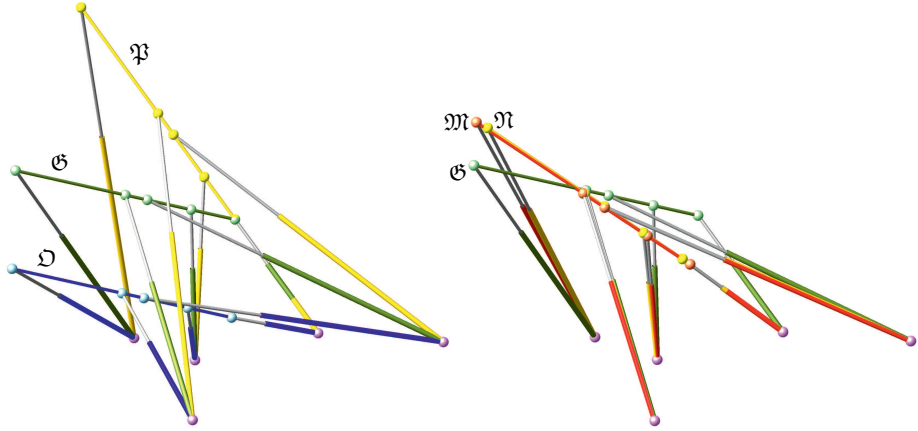

In contrast the determination of the closest singular pose (cf.\ Fig.\ \ref{fig:general}-right) within the complete configurations space (with respect to an object oriented metric) 
leads across the solution of a polynomial of degree 80. Due to this high degree a computation in real time is not possible. 
Our first idea to scope with this problem was to relax the motion group from the Euclidean one to the group of equiform motions 
(similarity transformations), 
which is equivalent to omitting the normalizing condition $\Gamma$. Doing so, the degree drops to 28, which was demonstrated in the addendum of \cite{arvin} 
and is displayed in Fig.\ \ref{fig:general}-right. 
As the obtained distance of the relaxed problem is less or equal to the distance of the original problem, it can be
used as the radius of a guaranteed singularity-free hypersphere. 

The designs computed in the Sections \ref{sec:2}--\ref{sec:4} imply a further degree reduction of 
the polynomials associated with the problem of determining singularity-free zones. 
This is demonstrated at the base of examples in Section \ref{sec:5}. 
Finally, the paper is concluded (cf.\ Section \ref{conclusion}) by proposing three kinematically redundant linear pentapods with a 
simplified singularity variety.

Before plunging into the computations behind the desired designs, clarifying the used notations seems necessary.

\subsection{Notation and preparatory work}
\label{notation}

The following notations are used in the rest of the paper: 
\begin{enumerate}
\item[$\bullet$]The compact notations
$\mathbf{X}=(x_{2}, x_{3}, x_{4}, x_{5})^{T}$, $\mathbf{Y}=(y_{2}, y_{3}, y_{4}, y_{5})^{T}$, 
$\mathbf{Z}=(z_{2}, z_{3}, z_{4}, z_{5})^{T}$  are introduced for the coordinates related to base anchor points. 
\item[$\bullet$]The compact notation 
$\mathbf{r}=(r_{2}, r_{3}, r_{4}, r_{5})^{T}$ is used for the coordinates related to platform anchor points. 
\item[$\bullet$] The \emph{component-wise} product of two vectors is given as follows:
\begin{eqnarray}
\mathbf{r}\mathbf{X}=
\left(\begin{array}{c}
       r_{2}x_{2} \\ 
       r_{3}x_{3} \\
       r_{4}x_{4} \\
       r_{5}x_{5} \\
\end{array}\right), \quad
\mathbf{r}\mathbf{Y}=
\left(\begin{array}{c}
       r_{2}y_{2} \\ 
       r_{3}y_{3} \\
       r_{4}y_{4} \\
       r_{5}y_{5} \\
\end{array}\right),\quad
\mathbf{r}\mathbf{Z}=
\left(\begin{array}{c}
       r_{2}z_{2} \\ 
       r_{3}z_{3} \\
       r_{4}z_{4} \\
       r_{5}z_{5} \\
\end{array}\right).
\end{eqnarray}
\item[$\bullet$]
For the sake of simplicity in notation as well as interpretation, we use the \emph{bracket}; i.e.:
\begin{eqnarray}
[\mathrm{A_{1}}, \mathrm{A_{2}}, \mathrm{A_{3}}, \mathrm{A_{4}}]=\det(\mathbf{A_{1}},\mathbf{A_{2}},\mathbf{A_{3}},\mathbf{A_{4}})
\quad\text{with}\quad \mathbf{A_{i}}\in\{ \mathbf{r},\mathbf{X}, \mathbf{Y}, \mathbf{Z}, \mathbf{rX}, \mathbf{rY},\mathbf{rZ}\}.
\label{BRACKET}
\end{eqnarray}
\end{enumerate}

Furthermore, a proper definition for \emph{undesired designs} or in another formidable word the \emph{architectural singularity} seems necessary.
 
\begin{Definition}
An "architectural singularity" refers to a robot design that is singular in all of its configurations. A robot possessing an architectural singularity is called an "architecturally singular manipulator".
\label{definition}
\end{Definition}

Equivalently a linear pentpod is an \emph{architecturally singular manipulator} if for every position and orientation, the matrix of Eq.\ (\ref{BorrasMatrix}) becomes rank deficient. 
By defining the \emph{architecture matrix} of linear pentapods, namely:
\begin{equation}
\mathbf{A}=\left( \begin {array}{ccccccc}
 r_{2} &x_{2} & y_{2} & z_{2} & r_{2}x_{2} & r_{2}y_{2} & r_{2}z_{2}
\\ r_{3}&x_{3}&y_{3}& z_{3} & r_{3}x_{3} & r_{3}y_{3} & r_{3}z_{3}
\\ r_{4}&x_{4}&y_{4
}&z_{4}&r_{4}x_{4}&r_{4}y_{4}&r_{4}z_{4}
\\ r_{5}&x_{5}&y_{5}&z_{5}&r_{5}x_{5}&r_
{5}y_{5}&r_{5}z_{5}\end {array} \right) 
\label{ArchMatrix}
\end{equation}
we can identify such singularities by considering the rank deficiency of this matrix.
\begin{lemma}
If the "architecture matrix" is rank deficient then the linear pentapod is an "architecturally singular manipulator".
\label{lemma:1}
\end{lemma}

\begin{proof}
Trivially if Eq.\ (\ref{ArchMatrix}) is rank deficient then the determinant of Eq.\ (\ref{BorrasMatrix}), which is the \emph{singularity polynomial}, vanishes.
\end{proof}

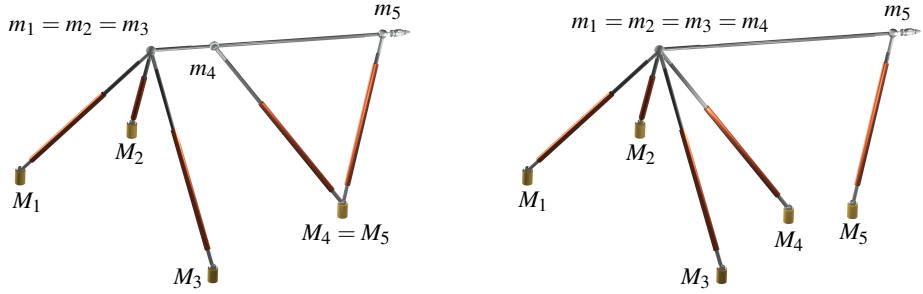
\begin{figure}[t!] 
\begin{center}   
  \begin{overpic}[height=35mm]{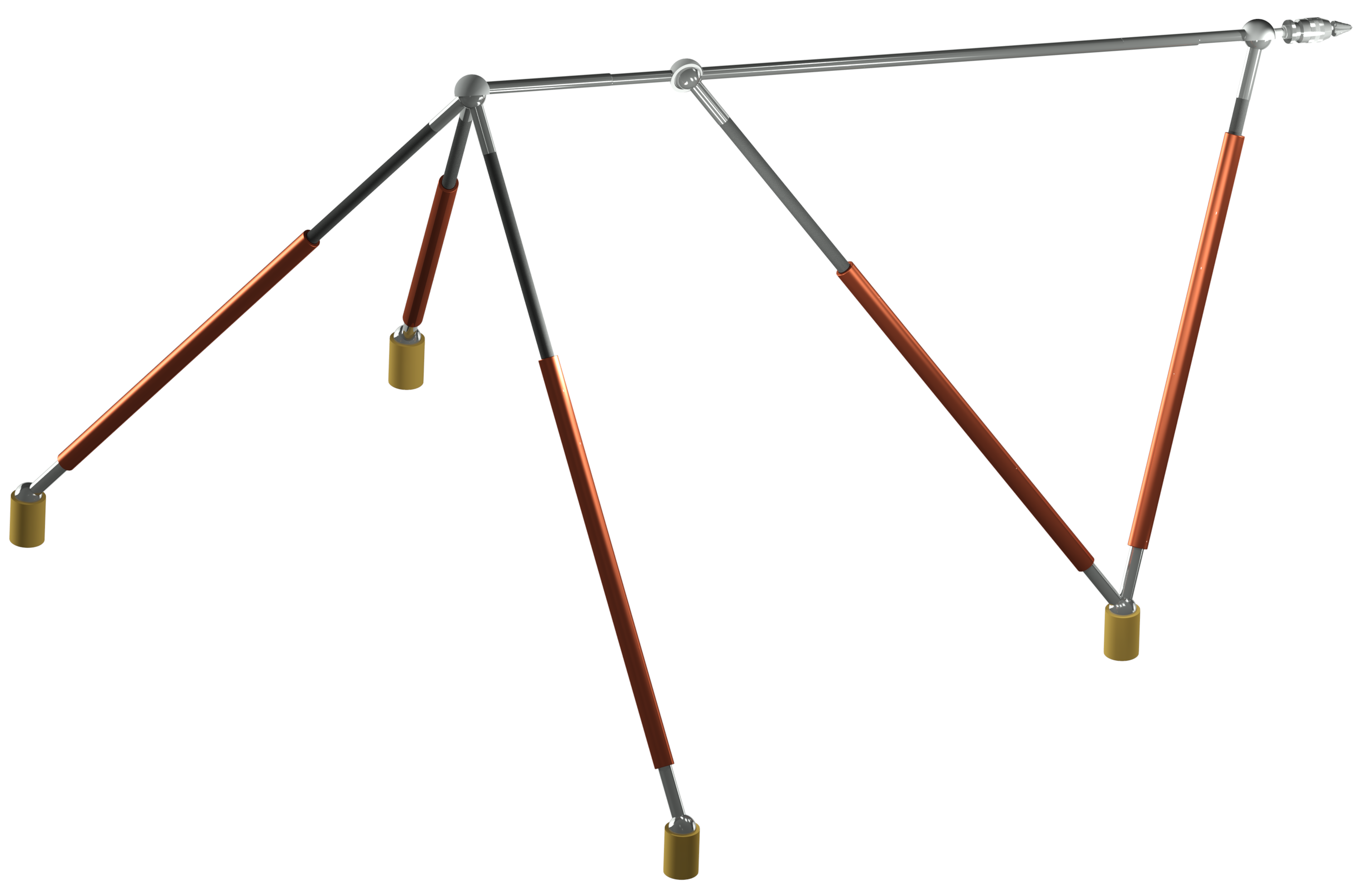}
	\begin{small}
	\put(-1,64){$m_{1}=m_{2}=m_{3}$}
   \put(44,53){$m_{4}$}
   \put(90,68){$m_{5}$}
   \put(0,20){$M_{1}$}
   \put(25.5,32){$M_{2}$}
   \put(40,1){$M_{3}$}
   \put(72,12){$M_{4}=M_{5}$}
	\end{small}     
  \end{overpic}
	\hfill
	\begin{overpic}[height=35mm]{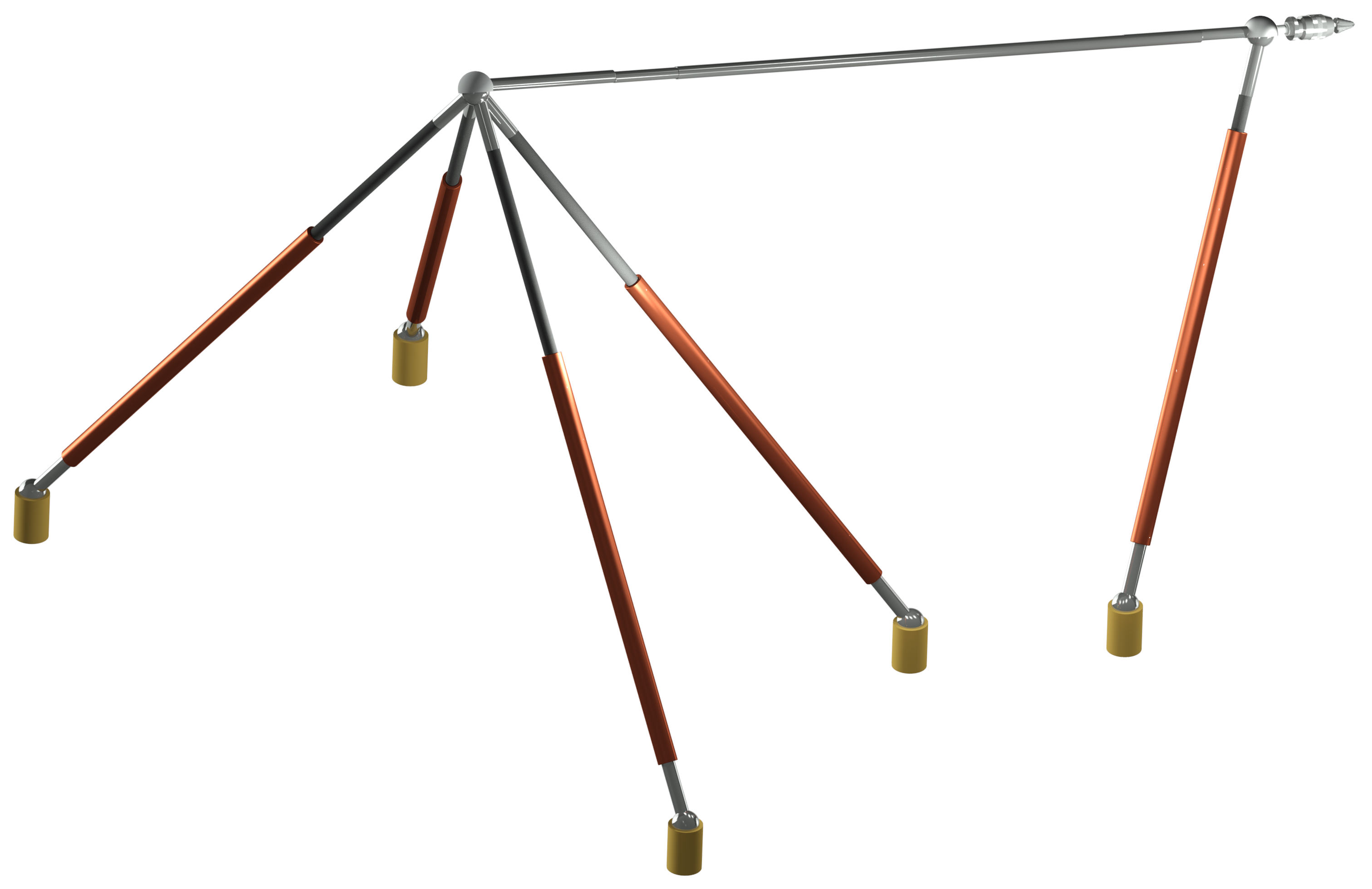}
	\begin{small}
	\put(13,64){$m_{1}=m_{2}=m_{3}=m_{4}$}
   \put(0,20){$M_{1}$}
   \put(26.5,31.5){$M_{2}$}
   \put(40,1){$M_{3}$}
   \put(63,10){$M_{4}$}
   \put(79,12){$M_{5}$}
   \put(90,67){$m_{5}$}
	\end{small}     
  \end{overpic}
	\caption{Illustrations of the counter examples for the necessity of the condition given in Lemma \ref{lemma:1}.}
	\label{fig:3}
\end{center}
\end{figure}

\begin{remark}
It is noteworthy that this is a sufficient but not necessary condition as it is well-known (cf.\ items (c) and (d) of Corollary 1 in \cite{nawratil2009newarch})
that there exist the following two\footnote{Up to renumbering  of the platform and base anchor points.} exceptional cases:
\begin{enumerate}[$\star$]
\item
$m_1=m_2=m_3$ and $M_4=M_5$,
\item
$m_1=m_2=m_3=m_4$,
\end{enumerate}
which are illustrated in Fig.\ \ref{fig:3}. \hfill $\diamond$
\end{remark}

Since in computational kinematics most of the computations are of symbolic type, and naturally expensive in the sense of time consumption, it will be highly favorable if we are able to eliminate some extra symbols.
The following lemma shows that it is possible to alleviate the burden of extra symbols in computations to come:

\begin{lemma}
If the linear pentapod is not an architecturally singular then there exists a triple of base points $M_{i}$, $M_{j}$ and $M_{k}$ which form a triangle and $m_{i}\neq m_{j}$ holds.
\label{lemma:2}
\end{lemma}

\begin{proof}
It is enough to show that there is a triangle where at least two of its corresponding platform points are not coinciding. First we claim that a triangle in the base always exists, as otherwise all five base points are collinear which yield a trivial 
architecture singular design. 
Now, since not all the platform points can collapse into a single point (if more than 3 platform anchor points coincide we get again a trivial 
architecture singular design) there should be at least two different points on the platform namely, $m_{i}$ and $m_{j}$. Now name the corresponding base points $M_{i}$ and $M_{j}$. If these two are not coincided then based on the first part of the proof it is possible to find another base point $M_{k}$ not co-linear with $M_{i}$, $M_{j}$ and hence the statement is fulfilled.

Now, suppose such a triangle with $m_{i}\neq m_{j}$ doesn't exist (see Fig.\ \ref{fig:4}). Then for $\triangle\  {M_{i}}{M_{k}}{M_{l}}$ and $\triangle\ M_{j}M_{k}M_{l}$ we have $m_{i}=m_{k}=m_{l}$ and $m_{j}=m_{k}=m_{l}$ respectively, which would yield $m_{i}=m_{j}$, a contradiction.  
\end{proof}

\begin{figure}[b!] 
\begin{center}   
  \begin{overpic}[height=34mm]{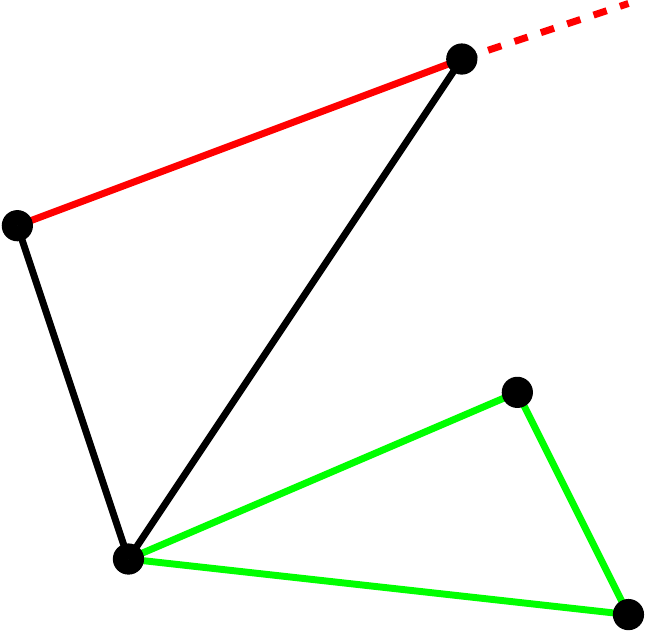}
	\begin{small}
	\put(0,70){${m_{i}}$}
	\put(73,75){${m_{j}}$}
	\put(-20,8){${M_{i}=M_{j}}$}
	\put(98,7){${M_{l}}$}
   \put(75,42){${M_{k}}$}
	\end{small}     
  \end{overpic} 
	\caption{The red line represents the linear pentapod's motion platform, while the green and black stand for the triangle in the base and the  legs, respectively.}
	\label{fig:4}
\end{center}
\end{figure}
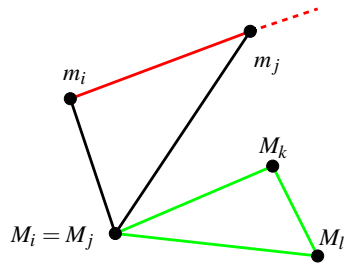

Based on this lemma one can assume ${M}_{1}=(0,0,0)$, ${M}_{2}=(x_{2},0,0)$ and ${M}_{3}=(x_{3},y_{3},0)$ where $x_{2} y_{3} \neq 0$. Moreover due to $m_{1}\neq m_{2}$ we can assume a scaling upon which, $r_{2}=1$ holds.
Now the \emph{architecture matrix} of Eq.\ (\ref{ArchMatrix}) simplifies into the following:
\begin{equation}
\mathbf{A}=\left( \begin {array}{ccccccc}
 1 &x_{2} & 0 & 0 & x_{2} & 0 & 0
\\ r_{3}&x_{3}&y_{3}& 0 & r_{3}x_{3} & r_{3}y_{3} & 0
\\ r_{4}&x_{4}&y_{4
}&z_{4}&r_{4}x_{4}&r_{4}y_{4}&r_{4}z_{4}
\\ r_{5}&x_{5}&y_{5}&z_{5}&r_{5}x_{5}&r_
{5}y_{5}&r_{5}z_{5}\end {array} \right). 
\label{ArchMatrix2}
\end{equation}
With the aid of Lemma \ref{lemma:2} and using projective geometry it is possible to obtain a simple but helpful geometric interpretation for \emph{architecturally singular} linear pentapods later in the coming sections. In fact one can think of 
$\mathbf{r}$, $\mathbf{X}$ and $\mathbf{rX}$ as points in the affine space $\mathbb{R}^{3}$ and the remaining columns of Eq.\ (\ref{ArchMatrix2}) as points on the plane at infinity $\Omega_{\infty}$, which closes $\mathbb{R}^{3}$ projectively; i.e. the columns 
of Eq.\ (\ref{ArchMatrix2}) can be seen as homogenous point coordinates of the 3-dimensional projective space $\mathbb{PR}^{3}$.

\begin{lemma}
The "architecture matrix" is rank deficient iff the points $\mathrm{r}$, $\mathrm{X}$, $\mathrm{Y}$, $\mathrm{Z}$, $\mathrm{rX}$, $\mathrm{rY}$ and $\mathrm{rZ}$ are co-planar in $\mathbb{PR}^{3}$.
\label{lemma:3}
\end{lemma}

\begin{proof}
A bracket defined in Eq.\ (\ref{BRACKET}) vanishes if and only if the four points in the bracket are co-planar in $\mathbb{PR}^{3}$ \citep{ben2008singulab}. Now, the \emph{architecture matrix} (a $4\times 7$ matrix) is rank deficient whenever all $4\times 4$ sub-matrices are of determinant zero. In another word the \emph{architecture matrix} is rank deficient iff any four members of the set $\{ \mathrm{r},\mathrm{X},\mathrm{Y},\mathrm{Z},\mathrm{rX},\mathrm{rY},\mathrm{rZ}\}$ are co-planar which happens if and only if these seven points are located on a common plane in $\mathbb{PR}^{3}$.
\end{proof}

Finally the computation of each case is based on the elimination of determinants of \emph{unwanted} sub-matrices of Eq.\ (\ref{BorrasMatrix}). These sub-matrices are named $\mathbf{S}_{\ i_{1},\ldots ,i_{n}}^{\ j_{1},\ldots ,j_{n}}$ where 
$i_{1},\ldots ,i_{n}$ indicates the numbers of the rows and $j_{1},\ldots ,j_{n}$ the numbers of the columns, which have to be 
deleted from the matrix given in Eq.\ (\ref{BorrasMatrix});  e.g. $\mathbf{S}_{\ 1,2,3}^{\ 4,5,6}$ stands for the sub-matrix obtained by removing the 1st, 2nd and 3rd row and 4th, 5th and 6th column.


\section{Linear in $p_{x}$, $p_{y}$ and $p_{z}$}
\label{sec:2}

In this section we determine all non-architectural singular designs, where the \emph{singularity polynomial} $det(\mathbf{S})=0$ 
is only linear in position variables. In the following we distinguish between linear pentapods with/without
coplanar base anchor points (planar/non-planar case).

\subsection{Planar case} 
\label{sec:2:1}

\begin{figure}[t!] 
\begin{center}   
  \begin{overpic}[height=48mm]{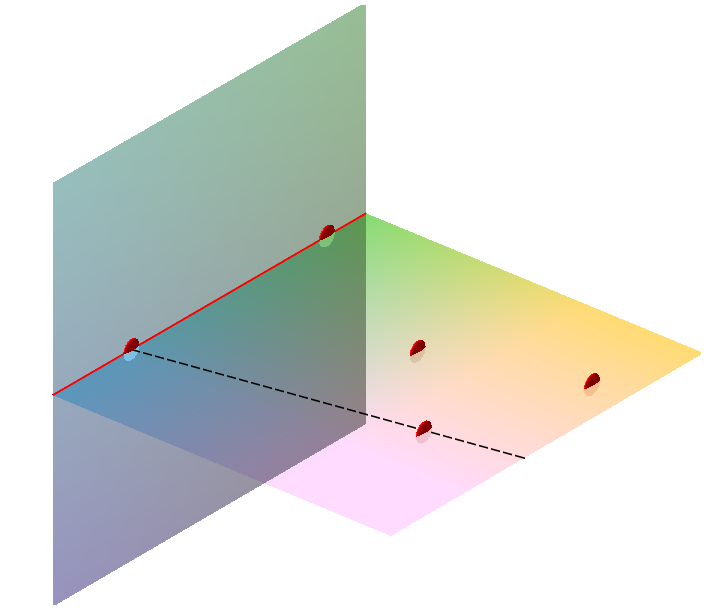}
	\put(17,42){$\mathrm{Y}$}
   \put(44,59){$\mathrm{rY}$}
   \put(58,18){$\mathrm{X}$}
   \put(59,41){$\mathrm{r}$}
   \put(82,36){$\mathrm{rX}$}
   \put(10,8){$\Omega_{\infty}$}
   \put(44,6){$\mathcal{P}_{1}=\mathcal{P}_{2}$}
  \end{overpic}
	\qquad
	\begin{overpic}[height=48mm]{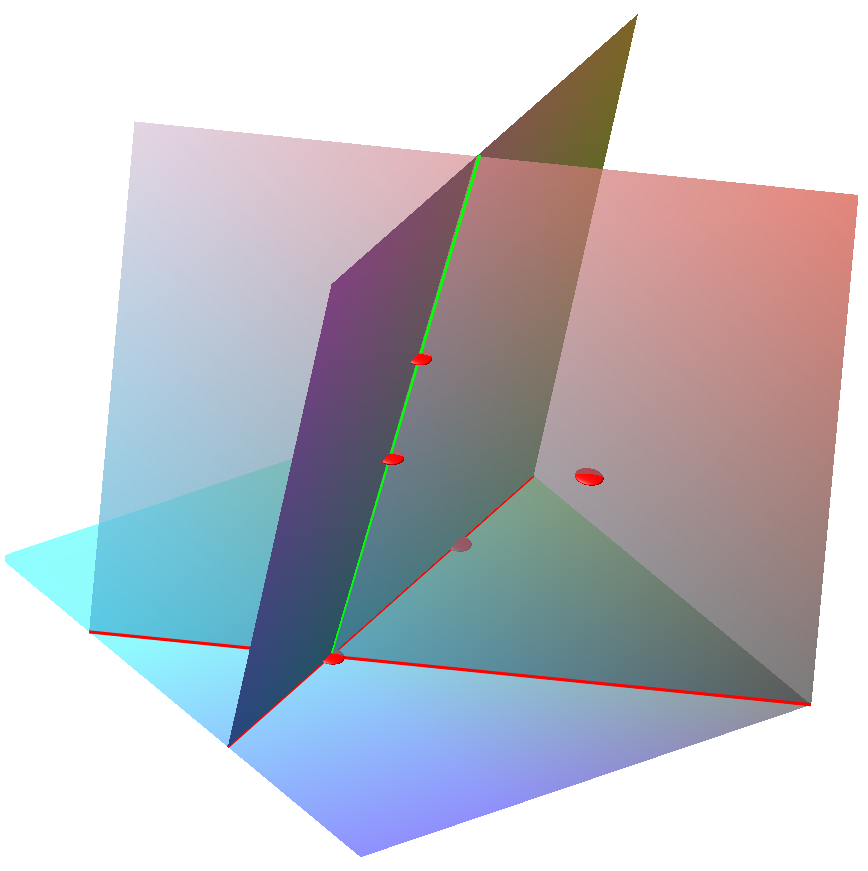}
	\put(32,24){$\mathrm{Y}$}
   \put(41,59){$\mathrm{X}$}
   \put(55,32){$\mathrm{rY}$}
   \put(39,46){$\mathrm{r}$}
   \put(72,43){$\mathrm{rX}$}
   \put(41,7){$\Omega_{\infty}$}
   \put(72,100){$\mathcal{P}_{1}$}
   \put(100,79){$\mathcal{P}_{2}$}
  \end{overpic}
	\caption{
	Geometric interpretation of the conditions yielding a \emph{singularity polynomial}, 
	which is linear in position variables for a linear pentapod with a planar base:  
	architecturally singular case (left) and the non-architecturally singular case (right).}
	\label{fig:5}
\end{center}
\end{figure}
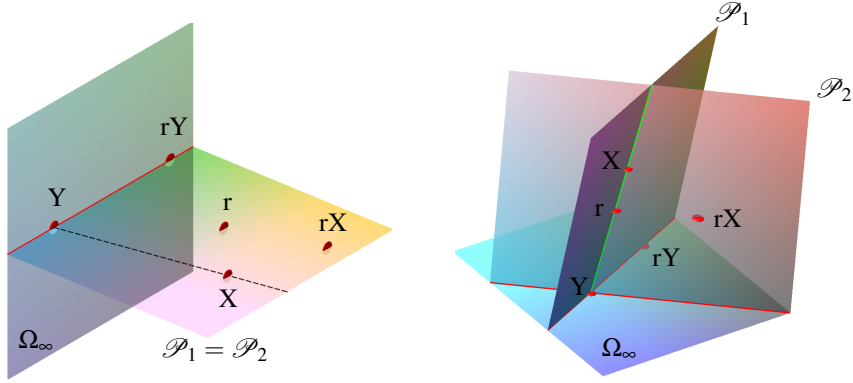     

Assume that the manipulator is planar ($z_{4}=z_{5}=0$). 
Since the desired goal here is to have the linear singularity polynomial in position variables, all the terms containing position variables of degree two should be canceled. These terms form a polynomial, 
which we call the \emph{undesired polynomial} through the remainder of the article. In a more general sense, the \emph{undesired polynomial} is a polynomial which by subtracting it from $\det (\mathbf{S})$ 
yields a polynomial with the desired property (this property can be linearity in position/orientation variables or quadratic in total).

Here the \emph{undesired polynomial} is as follows:
\begin{equation}
\det\left(\mathbf{S}_{\ 1,2}^{\ 7,4}\right)p_{z}^{2} + \det\left(\mathbf{S}_{\ 1,2}^{\ 5,4}\right) p_{x}p_{z} - \det\left(\mathbf{S}_{\ 1,2}^{\ 6,4}\right)p_{y}p_{z}=0. 
\label{undesired:1}
\end{equation}
If Eq.\ (\ref{undesired:1}) is fulfilled independently of the position variables then all the coefficients have to be zero. Based on the resulting conditions one can prove the following theorem:  

\begin{theorem}\label{thm:1}
A non-architecturally singular linear pentapod with a planar base can only have a "singularity polynomial", 
which is linear in position variables, 
iff there is a singular affine mapping $\kappa$ from the base plane to the platform line $\ell$ with $M_i\mapsto m_i$ for $i=1,\ldots ,5$.
\label{theorem:1}
\end{theorem}

\begin{proof}
Using \emph{Laplace expansion by minors}, $\det\left(\mathbf{S}_{\ 1,2}^{\ 7,4}\right)$ is:
\begin{equation}  
[\mathrm{r}, \mathrm{X}, \mathrm{Y}, \mathrm{rX}]v
-[\mathrm{r}, \mathrm{X}, \mathrm{Y}, \mathrm{rY}]u=0.
\label{up}
\end{equation} 
For all possible orientations, Eq.\ (\ref{up}) holds whenever both bracket coefficients vanish. Again by considering the \emph{Laplace expansion by minors} for $\det\left(\mathbf{S}_{\ 1,2}^{\ 5,4}\right)$ and $\det\left(\mathbf{S}_{\ 1,2}^{\ 6,4}\right)$ respectively, one obtains:
\begin{equation}
[\mathrm{r}, \mathrm{X}, \mathrm{Y}, \mathrm{rY}] w=[\mathrm{r}, \mathrm{X}, \mathrm{Y}, \mathrm{rX}] u=0. 
\label{prebracket2}
\end{equation}
As it is also desired to have these equations vanished for all possible orientations, the bracket coefficients should be equal to zero simultaneously. Hence, independently of all possible orientations, the following statement holds:
\begin{equation}
\det\left(\mathbf{S}_{\ 1,2}^{\ 7,4}\right)\ \text{vanishes} \Longleftrightarrow \det\left(\mathbf{S}_{\ 1,2}^{\ 5,4}\right)\ \text{and}\ \det\left(\mathbf{S}_{\ 1,2}^{\ 6,4}\right)\ \text{vanish}. 
\label{equation:1}
\end{equation}
Finally, based on Eq.\ (\ref{equation:1}) the necessary and sufficient condition for having a  \emph{singularity polynomial} linear in position variables will be:
\begin{equation}
[\mathrm{r}, \mathrm{X}, \mathrm{Y}, \mathrm{rY}]=[\mathrm{r}, \mathrm{X}, \mathrm{Y}, \mathrm{rX}]=0.
\label{bracket2}
\end{equation}
Using the literature of bracket algebra available at \cite{ben2008singulab, white1994grassmann} these brackets vanish whenever the four points characterizing them are co-planar. 
We denote the planes associated with the two brackets of Eq.\ (\ref{bracket2})-left and Eq.\ (\ref{bracket2})-right by $\mathcal{P}_{1}$ and $\mathcal{P}_{2}$, respectively. Then the following two cases 
have to be distinguished:  

\begin{itemize}
\item[1.] If the points $\mathrm{r}$, $\mathrm{X}$ and $\mathrm{Y}$ are not co-linear (or in another word if the vectors $\mathbf{r}$, $\mathbf{X}$ and $\mathbf{Y}$ are linearly independent) then the linear pentapod would be an \emph{architecturally singular manipulator} since geometrically, by Lemma \ref{lemma:3} this is equal to having the planes $\mathcal{P}_{1}$ and $\mathcal{P}_{2}$ coincided, as depicted in Fig.\ \ref{fig:5}-left.

\item[2.] If the points $\mathrm{r}$, $\mathrm{X}$ and $\mathrm{Y}$ are co-linear (or in another word if the vectors $\mathbf{r}$, $\mathbf{X}$ and $\mathbf{Y}$ are linearly dependent) then $\mathbf{r}\in\mathrm{span}\{\mathbf{X}, \mathbf{Y}\}$; i.e.\
\begin{equation}
\mathbf{r}=\alpha .\mathbf{X} + \beta .\mathbf{Y}, 
\label{res1}
\end{equation}
where $\alpha$ and $\beta$ are real numbers. This results the affine coupling $\kappa$ mentioned in Theorem \ref{theorem:1}, namely:
\begin{equation}
\kappa:\,\,(x_{i},y_{i}) \longmapsto r_{i}=\alpha x_{i} + \beta y_{i} \quad\text{with}\quad \alpha =\tfrac{1}{x_{2}}\quad
\text{and}\quad i=2,\ldots , 5.
\label{affine}
\end{equation}
Geometrically, it is also worth mentioning that the planes $\mathcal{P}_{1}$ and $\mathcal{P}_{2}$ do not necessarily coincide in this case (as depicted in Fig.\ \ref{fig:5}-right).
\end{itemize}
\end{proof}

\begin{remark}
Now this question arises that which designs possessing a singular affine coupling $\kappa:$ $M_1\mapsto m_1$ are architecturally singular. 
According to the list given in \cite[Corollary 1]{nawratil2015self} there are three possible types of "architecturally singular manipulators", \footnote{Up to renumbering of the platform and base anchor points.} 
which will be as follows: 
\begin{enumerate}[$\star$]
\item
$m_1,\ldots ,m_5$ are pairwise distinct and there exists a conic section passing through $M_1,\ldots ,M_5$ and the ideal point of the 
parallel fibers of $\kappa$ (cf.\ Section 4.3 of \cite{borras2011architectural}). This is for example trivially fulfilled if $M_1,\ldots ,M_4$ are collinear.
\item
$m_1,\ldots ,m_4$ are pairwise distinct, $m_4=m_5$ holds and the base points $M_1,M_2,M_3$ are collinear. 
\item
More than two platform points coincide. \hfill $\diamond$
\end{enumerate}
\end{remark}

\subsection{Non-planar case}
\label{sec:2:2}

If the base points of the linear pentapod are not restricted to be positioned on a plane, the coordinates $z_{4}$ and $z_{5}$ can not vanish simultaneously. Making this assumption, the \emph{undesired polynomial} reads as follows:
\begin{equation*}
{\det} \left( \mathbf{S}_{\ 1,2}^{\ 5,2} \right) {p_{{x}}}^{2}+{\det}
 \left( \mathbf{S}_{\ 1,2}^{\ 6,3} \right) {p_{{y}}}^{2}+{\det} \left( \mathbf{S}_{\ 1,2}^{\ 7,4} \right) {p_{{z}}}^{2} - \left[ {\det} \left( \mathbf{S}_{\ 1,2}^{\ 5,3} \right) +{\det}
 \left( \mathbf{S}_{\ 1,2}^{\ 6,2} \right)  \right] p_{{x}}p_{{y}}
\end{equation*}
\begin{equation}
+ \left[ {\det} \left( \mathbf{S}_{\ 1,2}^{\ 5,4} \right) + {\det} \left( \mathbf{S}_{\ 1,2}^{\ 7,2} 
 \right)  \right] p_{{x}}p_{{z}} - \left[ {\det} \left( \mathbf{S}_{\ 1,2}^{\ 6,4} \right) +{\det} \left( \mathbf{S}_{\ 1,2}^{\ 7,3} \right)  \right] p_{{y}}p_{{z}}=0.
\label{undesired:2}
\end{equation} 
Again all sub-matrices appearing as the coefficients in Eq.\ (\ref{undesired:2}) should become rank deficient. 

\begin{theorem}
Non-architecturally singular linear pentapods with a non-planar base possessing a "singularity polynomial", 
which is linear in position variables, do not exist. 
\label{theorem:2}
\end{theorem}

\begin{proof}
Eq.\ (\ref{undesired:2}), independently of the position variables, gives $\det\left( \mathbf{S}_{\ 1,2}^{\ 5,2}\right)=\det\left(\mathbf{S}_{\ 1,2}^{\ 6,3} \right)=0$. Now, using \emph{Laplace expansion by minors} for $\det\left( \mathbf{S}_{\ 1,2}^{\ 5,2}\right)$ and $\det\left(\mathbf{S}_{\ 1,2}^{\ 6,3} \right)$ one finds:
\begin{equation}
[\mathrm{r}, \mathrm{Y}, \mathrm{Z}, \mathrm{rZ}]=[\mathrm{r}, \mathrm{Y}, \mathrm{Z}, \mathrm{rY}]=0, 
\label{bracket:3}
\end{equation}
\begin{equation}
[\mathrm{r}, \mathrm{X}, \mathrm{Z}, \mathrm{rZ}]=[\mathrm{r}, \mathrm{X}, \mathrm{Z}, \mathrm{rX}]=0.
\label{bracket:4}
\end{equation}
Now, it is possible to deduce the following:
\begin{itemize}
\item[1.] From Eq.\ (\ref{bracket:3})-left:
\begin{equation}
\mathbf{rZ}\in \mathrm{span}\{ \mathbf{r} , \mathbf{Y} , \mathbf{Z} \}, 
\label{item1}
\end{equation}
as $\mathbf{r}$, $\mathbf{Y}$ and $\mathbf{Z}$ are obviously linearly independent.
\item[2.] From Eq.\ (\ref{bracket:3})-right one derives:
\begin{equation}
\mathbf{rY}\in \mathrm{span}\{ \mathbf{r} , \mathbf{Y} , \mathbf{Z} \}.
\label{item2}
\end{equation}
\item[3.] 
If we replace in Eq.\ (\ref{bracket:4})-left the expression $\mathbf{rZ}$ by the linear combination resulting from  Eq.\ (\ref{item1}) we get: 
\begin{equation}
\mathbf{X}\in \mathrm{span}\{ \mathbf{r} , \mathbf{Y} , \mathbf{Z} \}.
\label{item3}
\end{equation}
\item[4.] From Eq.\ (\ref{bracket:4})-right and Eq.\ (\ref{item3}) one finds:
\begin{equation}
\mathbf{rX}\in \mathrm{span}\{ \mathbf{r} , \mathbf{Y} , \mathbf{Z} \}.
\label{item4}
\end{equation}
\end{itemize}
Now, using Eqs.\ (\ref{item1}--\ref{item4}), 4 out of 7 columns of the \emph{architecture matrix} are linearly dependent and thus rank deficient: 
\begin{equation}
\mathrm{Rank}(\mathbf{r}, \mathbf{X}, \mathbf{Y}, \mathbf{Z}, \mathbf{r}\mathbf{X}, \mathbf{rY}, \mathbf{rZ})<4.
\label{architecture1}
\end{equation}
\end{proof}

\section{Linear in $u$, $v$ and $w$}
\label{sec:3}

In this section we determine all non-architecturally singular designs where the \emph{singularity polynomial} $det(\mathbf{S})=0$ 
is only linear in orientation variables. 
As in Section \ref{sec:2} we distinguish between linear pentapods with planar and non-planar bases.

\subsection{Planar case}
\label{sec:3:1}

Under the planar condition ($z_{4}=z_{5}=0$) the \emph{undesired polynomial} is:
\begin{equation*}
\left[\det\left( \mathbf{S}_{\ 1,3}^{\ 3,7}\right) + \det\left( \mathbf{S}_{\ 1,3}^{\ 4,6}\right) \right] vw + \left[ \det\left( \mathbf{S}_{\ 1,3}^{\ 2,7}\right) - \det\left( \mathbf{S}_{\ 1,3}^{\ 4,5}\right) \right] uw + 
\end{equation*}
\begin{equation}
\det\left( \mathbf{S}_{\ 1,3}^{\ 2,6}\right)uv -\det\left( \mathbf{S}_{\ 1,3}^{\ 4,7}\right) {w}^{2}=0.
\label{undesired:3}
\end{equation}
\begin{theorem}\label{thm:orientation}
A non-architecturally singular linear pentapod with a planar base can only have a "singularity polynomial", 
which is linear in orientation variables, in the following cases:
\begin{enumerate}
\item
$M_2$, $M_3$, $M_4$, $M_5$ are collinear, 
\item
$m_1=m_i$ and  $M_j$, $M_k$, $M_l$ are collinear with pairwise distinct $i,j,k,l\in\left\{2,3,4,5\right\}$,
\item
 $m_1=m_i=m_j$ with pairwise distinct $i,j\in\left\{2,3,4,5\right\}$.
\end{enumerate}
\label{theorem:3}
\end{theorem}
\begin{proof}
Eq.\ (\ref{undesired:3}), independently of the orientations variables, gives $\det\left( \mathbf{S}_{\ 1,3}^{\ 4,7}\right)=0$. By resorting to the literature of \emph{brackets}, $\det\left( \mathbf{S}_{\ 1,3}^{\ 4,7}\right)=0$ if and only if the following holds:
\begin{equation}
[\mathrm{r}, \mathrm{Y}, \mathrm{rX}, \mathrm{rY}]=[\mathrm{r}, \mathrm{X}, \mathrm{rX}, \mathrm{rY}]=0.
\label{bracket:5}
\end{equation}
Now, name the plane characterized by the points $\mathrm{r}$, $\mathrm{rX}$ and $\mathrm{rY}$ as $\mathcal{P}_{1}$. If the points $\mathrm{r}$, $\mathrm{rX}$ and $\mathrm{rY}$ are not co-linear then the plane $\mathcal{P}_{1}$ is defined \emph{uniquely} and hence by Eq.\ (\ref{bracket:5}) $\mathrm{X}$ and $\mathrm{Y}$ are also on $\mathcal{P}_{1}$ which by Lemma \ref{lemma:3} results in a rank deficiency of the \emph{architecture matrix} (cf. Fig.\ \ref{fig:6}-left).

\begin{figure}[t!] 
\begin{center}   
  \begin{overpic}[height=48mm]{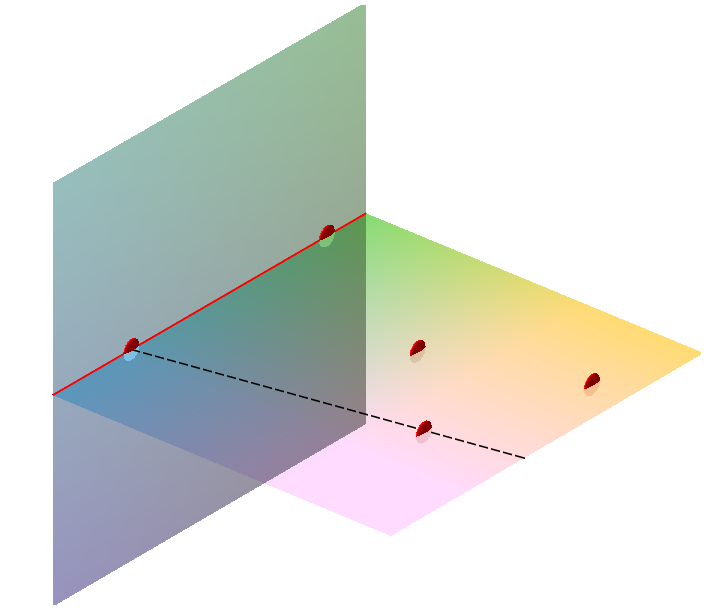}
	\put(17,42){$\mathrm{rY}$}
   \put(44,59){$\mathrm{Y}$}
   \put(58,18){$\mathrm{rX}$}
   \put(59,41){$\mathrm{r}$}
   \put(82,36){$\mathrm{X}$}
   \put(10,8){$\Omega_{\infty}$}
   \put(44,6){$\mathcal{P}_{1}=\mathcal{P}_{2}$}
  \end{overpic}
	\hfill
	\begin{overpic}[height=48mm]{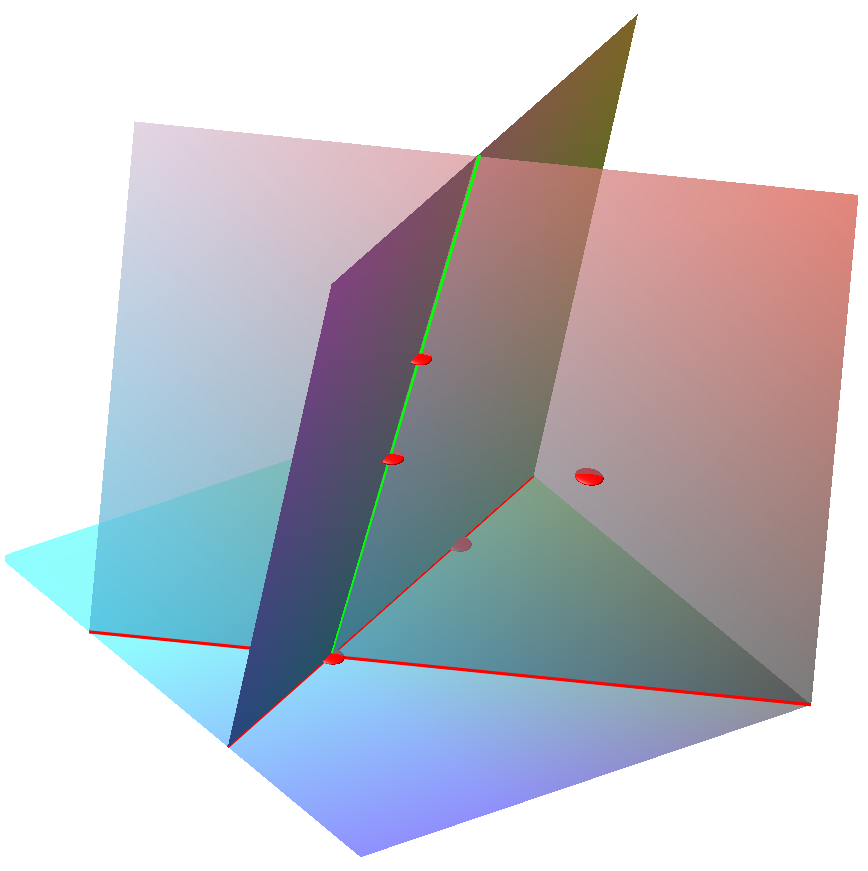}
	\put(31,24){$\mathrm{rY}$}
    \put(39,59){$\mathrm{rX}$}
    \put(55,32){$\mathrm{Y}$}
    \put(39,46){$\mathrm{r}$}
    \put(72,43){$\mathrm{X}$}
    \put(41,7){$\Omega_{\infty}$}
    \put(72,100){$\mathcal{P}_{1}$}
    \put(100,79){$\mathcal{P}_{2}$}
    \put(55.5,71){$\mathtt{l}$}
  \end{overpic}
	\caption{
	Geometric interpretation of the conditions yielding a \emph{singularity polynomial}, 
	which is linear in orientation variables for a linear pentapod with a planar base:
	architecturally singular case (left) and the non-architecturally singular case (right).}
	\label{fig:6}
\end{center}
\end{figure}     

On the other hand if the points $\mathrm{r}$, $\mathrm{rX}$ and $\mathrm{rY}$ are co-linear then there is the possibility of having the points $\mathrm{X}$ and $\mathrm{Y}$ on two different planes, namely $\mathcal{P}_{1}$ and $\mathcal{P}_{2}$ as depicted in Fig.\ \ref{fig:6}-right, which does not necessarily lead to an \emph{architectural singularity}. 
Under this assumption, we get $\mathbf{r}\in\mathrm{span}\{\mathbf{rX}, \mathbf{rY}\}$; i.e.\
\begin{equation}
\mathbf{r}=\alpha .\mathbf{rX} + \beta .\mathbf{rY}, 
\label{equation:2}
\end{equation}
where $\alpha$ and $\beta$ are real numbers with $(\alpha, \beta)\neq (0,0)$. Now having Eq.\ (\ref{equation:2}) in mind, 
the following possibilities arise (cf. Fig.\ \ref{fig:7}):
\begin{itemize}
\item[1.] $\forall\ i\in \{2,...,5\}, r_{i}\ \neq\ 0$. This yields: 
\begin{equation}
\mathbf{1}=\alpha.\mathbf{X}+\beta.\mathbf{Y}.
\label{line:point}
\end{equation}
Geometrically in this case, the point $\mathbf{1}=(1,1,1,1)$ should always be on the line $\mathtt{l}$ defined by the two points $\mathrm{X}$ and $\mathrm{Y}$. Moreover Eq.\ (\ref{line:point}) gives:
\begin{equation*}
 \Bigg\{ \begin{array}{ccc}
  x_{2}=\frac{1}{\alpha} & \ &  \\
  \ & \ & \ \\
  \alpha x_{i} +\beta y_{i} = 1 & \ & \textrm{for i}> 2, 
  \end{array}
\end{equation*} 
which means the base points $M_2$, $M_3$, $M_4$, $M_5$, are collinear.
\item[2.] $\exists !\ i \in\{3,4,5\}\ \text{such that}\ r_{i}=0$. Geometrically this means that one of the points $(1,0,1,1)^{T}$,$(1,1,0,1)^{T}$ or $(1,1,1,0)^{T}$ should be on $\mathtt{l}$. Naturally this yields $m_1=m_i$ and  $M_j$ and $M_k$ are collinear with pairwise distinct $i,j,k\in\left\{3,4,5\right\}$.
\item[3.] $\exists\ i\ \text{and}\ j \in\{3,4,5\},\ \text{where}\ i\neq j\ \text{such that}\ r_{i}=r_{j}=0 $. Geometrically this means that only one of the points $(1,0,0,1)^{T}$,$(1,0,0,1)^{T}$ or $(1,0,1,0)^{T}$ can be on $\mathtt{l}$ which yields $m_1=m_i=m_j$ with pairwise distinct $i,j\in\left\{3,4,5\right\}$.
\end{itemize}
\end{proof}

\begin{figure}[t!] 
\begin{center}   
  \begin{overpic}[height=25.7mm]{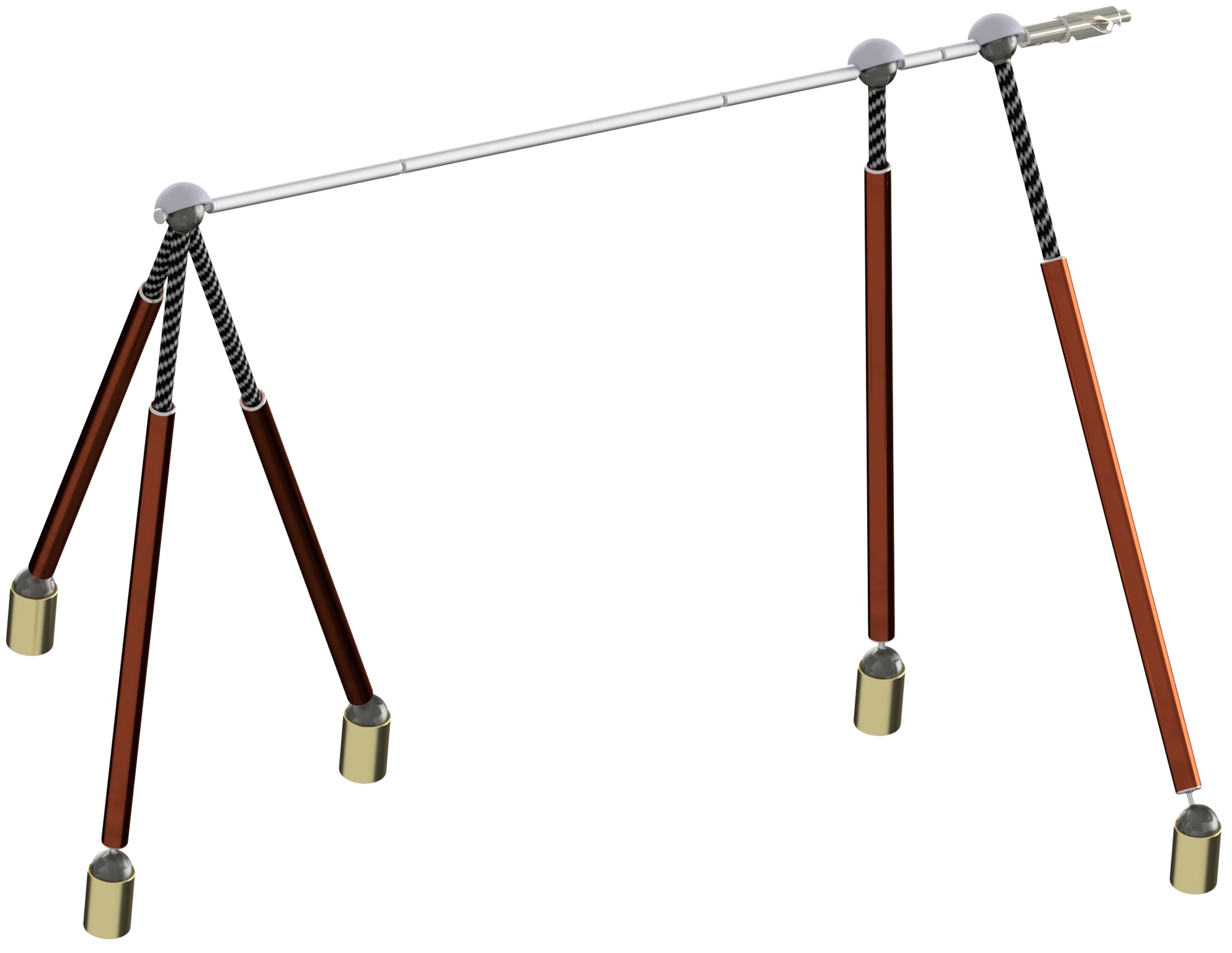}
  \end{overpic}
	\hfill
	\begin{overpic}[height=25.7mm]{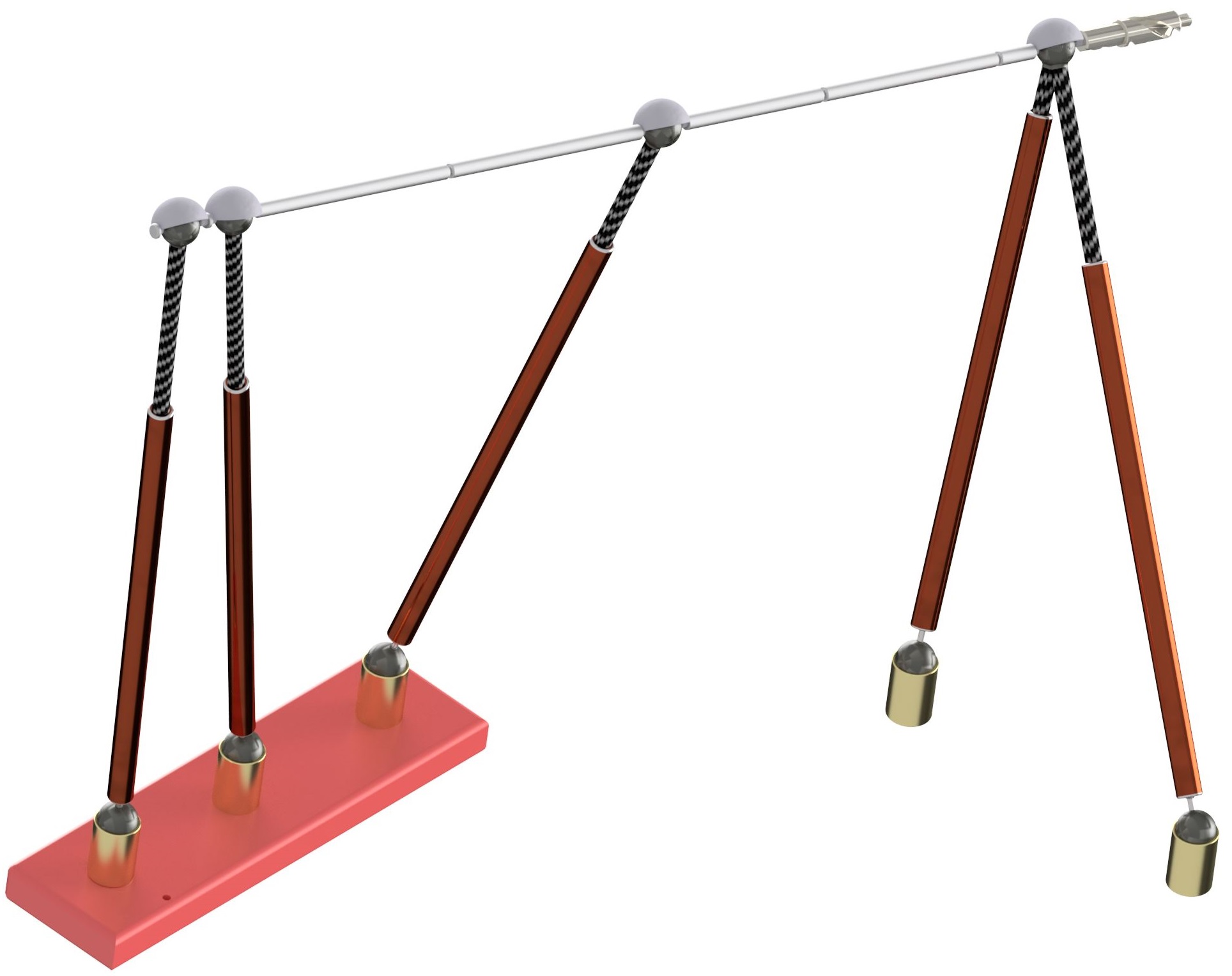}
  \end{overpic}
  \hfill
	\begin{overpic}[height=25.7mm]{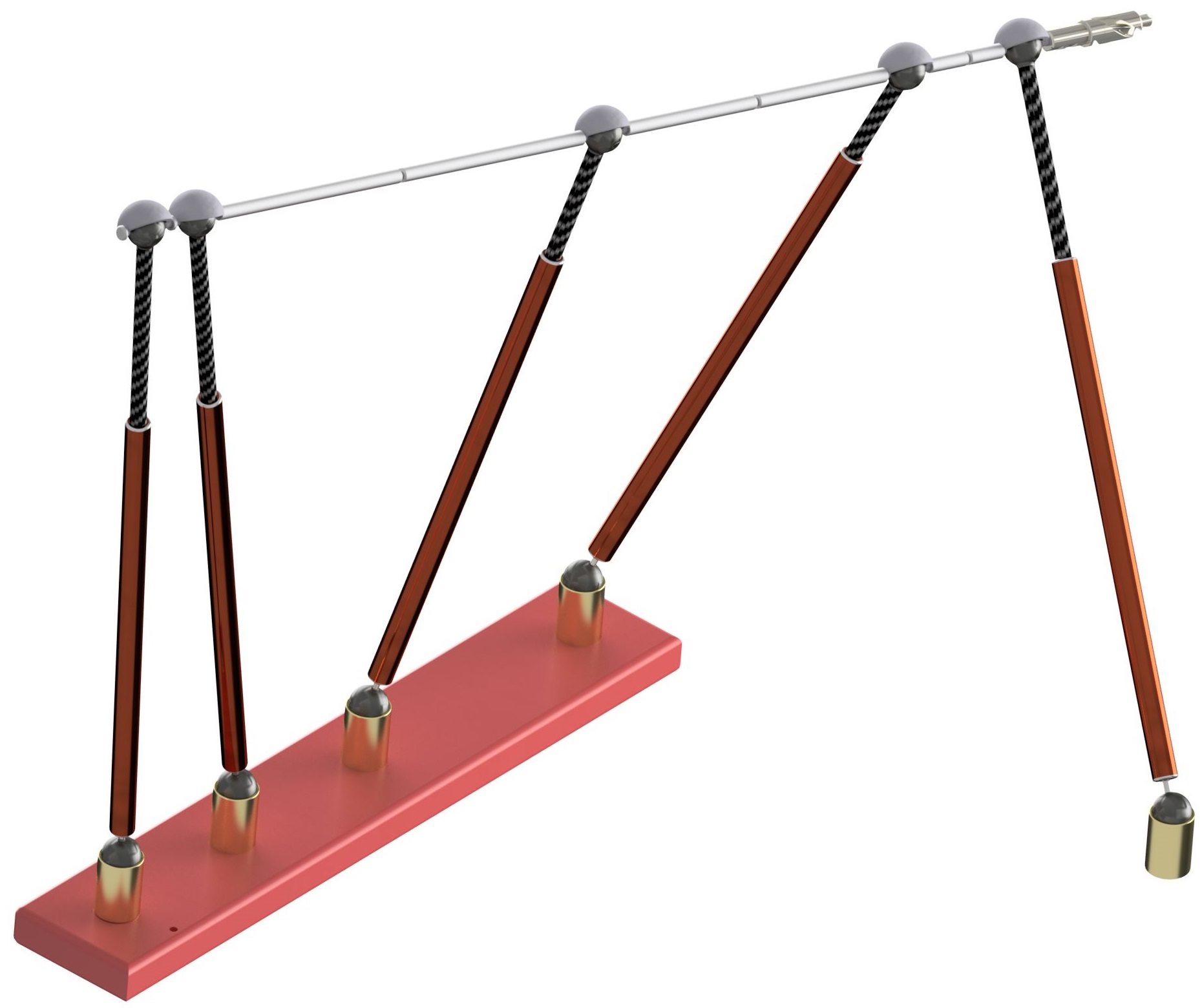}
  \end{overpic}
	\caption{
	Three possible designs mentioned in Theorem \ref{theorem:3}. Note that in the middle figure three base points are co-linear while in the right figure four base points are located on a line. The co-linear base points are indicated by red beams.}
	\label{fig:7}
\end{center}
\end{figure}  

\begin{remark}
Note that it is impossible to have the point $(1,0,0,0)^{T}$ on the line $\mathtt{l}$, since this would mean that 
the four platform anchor points should be coincided, which yields an "architecturally singular manipulator". For the additional conditions on the  designs 1--3 of Theorem \ref{thm:orientation} rendering the manipulator architecture singular we refer to the list given in \cite[Corollary 1]{nawratil2015self}. \hfill $\diamond$
\end{remark}

\subsection{Non-planar case}
\label{sec:3:2}

Based on the desired non-planarity condition ($z_{4}\neq 0$ or $z_{5}\neq 0$) and linearity in orientation variables the \emph{undesired polynomial} is as follows:
\begin{equation*}
{\det} \left( \mathbf{S}_{\ 1,3}^{\ 2,5} \right) {u}^{2}+{\det}
 \left( \mathbf{S}_{\ 1,3}^{\ 3,6} \right) {v}^{2}+{\det} \left( \mathbf{S}_{\ 1,3}^{\ 4,7} \right) {w}^{2}- \left[ {\det} \left( \mathbf{S}_{\ 1,3}^{\ 3,7}\right) 
+ {\det} \left( \mathbf{S}_{\ 1,3}^{\ 4,6} \right)  \right] vw \ 
\end{equation*}
\begin{equation}
- \left[ {\det} \left( \mathbf{S}_{\ 1,3}^{\ 2,6} \right) +{\det} \left( \mathbf{S}_{\ 1,3}^{\ 3,5} \right)  \right] uv+ \left[ {\det} \left( \mathbf{S}_{\ 1,3}^{\ 2,7} \right) +{\det} \left( \mathbf{S}_{\ 1,3}^{\ 4,5} \right)  \right] uw = 0.
\label{undesired:4}
\end{equation}

\begin{theorem}
Non-architecturally singular linear pentapods with a non-planar base possessing a \emph{singularity polynomial}, 
which is linear in orientation variables, do not exist. 
\end{theorem}
\begin{proof}
Eq.\ (\ref{undesired:4}), independently of orientation variables, gives ${\det} \left( \mathbf{S}_{\ 1,3}^{\ 2,5} \right)={\det}\left( \mathbf{S}_{\ 1,3}^{\ 3,6} \right)={\det} \left( \mathbf{S}_{\ 1,3}^{\ 4,7} \right)=0$. By using \emph{Laplace expansion by minors} for ${\det} \left( \mathbf{S}_{\ 1,3}^{\ 2,5} \right)$, ${\det}
 \left( \mathbf{S}_{\ 1,3}^{\ 3,6} \right)$ and ${\det} \left( \mathbf{S}_{\ 1,3}^{\ 4,7} \right)$ one finds:
\begin{equation}
[\mathrm{r}, \mathrm{Y}, \mathrm{rY}, \mathrm{rZ}]=[\mathrm{r}, \mathrm{Y}, \mathrm{rX}, \mathrm{rY}]=0, 
\label{bracket:7}
\end{equation}
\begin{equation}
\ [\mathrm{r}, \mathrm{Z}, \mathrm{rX}, \mathrm{rZ}]=[\mathrm{r}, \mathrm{X}, \mathrm{rX}, \mathrm{rZ}]=0.
\label{bracket:8}
\end{equation} 
Now from Eq.\ (\ref{bracket:7}) and Eq.\ (\ref{bracket:8}) the following is argued: 
\begin{itemize}
\item[1.] From Eq.\ (\ref{bracket:7})-left:
\begin{equation}
\mathbf{rY}\in \mathrm{span}\{ \mathbf{r}, \mathbf{Y}, \mathbf{rZ}\}.
\label{it1}
\end{equation}
\item[2.] From Eq.\ (\ref{bracket:7})-right and Eq.\ (\ref{it1}):
\begin{equation}
\mathbf{rX}\in \mathrm{span}\{ \mathbf{r}, \mathbf{Y}, \mathbf{rZ}\}.
\label{it2}
\end{equation} 
\item[3.] From Eq.\ (\ref{bracket:8})-left and Eq.\ (\ref{it2}): 
\begin{equation}
\mathbf{Z}\in \mathrm{span}\{ \mathbf{r}, \mathbf{Y}, \mathbf{rZ}\}.
\label{it3}
\end{equation}
\item[4.] From Eq.\ (\ref{bracket:8})-right and Eq.\ (\ref{it2}):
\begin{equation}
\mathbf{X}\in \mathrm{span}\{ \mathbf{r}, \mathbf{Y}, \mathbf{rZ}\}.
\label{it4}
\end{equation}
\end{itemize}

Now, considering Eqs.\ (\ref{it1}--\ref{it4}) one would obtain
\begin{equation}
\mathrm{Rank}(\mathbf{r}, \mathbf{X}, \mathbf{Y}, \mathbf{Z}, \mathbf{r}\mathbf{X}, \mathbf{rY}, \mathbf{rZ})<4.
\end{equation}
which implies an \emph{architecturally singular manipulator}.
\end{proof}

\section{Quadratic}
\label{sec:4}

In this section we study linear pentapods where the \emph{singularity polynomial} is only quadratic in total. Unfortunately we are only able to report the following negative result:

\begin{theorem}
Non-architecturally singular linear pentapods possessing a \emph{singularity polynomial}, 
which is quadratic in pose variables, do not exist. 
\end{theorem}

\begin{proof}
We can separate the proof into two parts: planar case and non-planar case.

The \emph{undesired polynomial} in the planar case ($z_{4}=z_{5}=0$) of quadratic \emph{singularity polynomial} is:
\begin{equation*}
\left[{\det} \left( \mathbf{S}_{\ 1,2,3}^{\ 7,4,5} \right){p_{z}}^{2} - {\det} \left( \mathbf{S}_{\ 1,2,3}^{\ 2,4,7} \right) w p_{z}\right] u+ \left[ {\det} \left( \mathbf{S}_{\ 1,2,3}^{\ 3,4,7} \right) w p_{{z}} - {\det} \left( \mathbf{S}_{\ 1,2,3}^{\ 7,4,6} \right){p_{z}}^{2}\right] v\ + 
\end{equation*}
\begin{equation}
\left[ {\det} \left( \mathbf{S}_{\ 1,2,3}^{\ 4,2,7} \right) p_{x} - {\det} \left( \mathbf{S}_{\ 1,2,3}^{\ 4,3,7} \right) p_{y}\right] {w}^{2} + \left[{\det} \left( \mathbf{S}_{\ 1,2,3}^{\ 6,4,7} \right) p_{y} p_{z} - {\det} \left( \mathbf{S}_{\ 1,2,3}^{\ 5,4,7} \right) p_{x} p_{z}\right] w=0.
\label{undesired:5}
\end{equation}
It would be noteworthy that the above polynomial is in fact ${\det} \left( \mathbf{S}_{\ 1}^{\ 1} \right)$. In another word the cubic part of the \emph{singularity polynomial} is encoded in this sub-matrix.
Eq.\ (\ref{undesired:5}) implies ${\det} \left( \mathbf{S}_{\ 1,2,3}^{\ 2,4,7} \right)={\det} \left( \mathbf{S}_{\ 1,2,3}^{\ 3,4,7} \right)={\det} \left( \mathbf{S}_{\ 1,2,3}^{\ 6,4,7} \right)={\det} \left( \mathbf{S}_{\ 1,2,3}^{\ 5,4,7} \right)=0$, which upon transforming into brackets will give: 
\begin{equation}
[\mathrm{r}, \mathrm{Y}, \mathrm{rX}, \mathrm{rY}]=[\mathrm{r}, \mathrm{X}, \mathrm{rX}, \mathrm{rY}]=0,
\label{bracket:9}
\end{equation} 
\begin{equation}
[\mathrm{r}, \mathrm{X}, \mathrm{Y}, \mathrm{rY}]\ =\ [\mathrm{r}, \mathrm{X}, \mathrm{Y}, \mathrm{rX}]\ =0.
\label{bracket:10}
\end{equation} 
By resorting to the geometrical interpretation of brackets, Eqs.\ (\ref{bracket:9}, \ref{bracket:10}) show that the points $\mathrm{X}$ and $\mathrm{Y}$ are on a plane characterized by the points $\mathrm{r}$, $\mathrm{rX}$ and $\mathrm{rY}$ which results in an \emph{architectural singularity}.

Under the non-planar condition ($z_{4}\neq 0$ or $z_{5}\neq 0$) the \emph{undesired polynomial} is as follows:
\begin{equation*}
\Big(\Big[ {\det} \left( \mathbf{S}_{\ 1,2,3}^{\ 2,3,5} \right)u + {\det} \left( \mathbf{S}_{\ 1,2,3}^{\ 2,3,6} \right)v - {\det} \left( \mathbf{S}_{\ 1,2,3}^{\ 2,3,7} \right) w \Big] p_{{y}} + \Big[ {\det} \left( \mathbf{S}_{\ 1,2,3}^{\ 2,4,5} \right)u - {\det} \left( \mathbf{S}_{\ 1,2,3}^{\ 2,4,6} \right)v
\end{equation*}
\begin{equation*}
+ {\det} \left( \mathbf{S}_{\ 1,2,3}^{\ 2,4,7} \right) w \Big] p_{{z}} \Big)u + \Big( - \left[ {\det} \left( \mathbf{S}_{\ 1,2,3}^{\ 3,2,5} \right) u - {\det} \left( \mathbf{S}_{\ 1,2,3}^{\ 3,2,6} \right) v + {\det} \left( \mathbf{S}_{\ 1,2,3}^{\ 3,2,7} \right) w \right] p_{{x}} 
\end{equation*}
\begin{equation*}
+ \Big[{\det} \left( \mathbf{S}_{\ 1,2,3}^{\ 3,4,5} \right) u
-{\det} \left( \mathbf{S}_{\ 1,2,3}^{\ 3,4,6} \right) v + {\det} \left( \mathbf{S}_{\ 1,2,3}^{\ 3,4,7} \right) w \Big] p_{{z}} \Big) v + \Big( \Big[ {\det} \left( \mathbf{S}_{\ 1,2,3}^{\ 4,2,5} \right) u   
\end{equation*}
\begin{equation*} 
+ {\det} \left( \mathbf{S}_{\ 1,2,3}^{\ 4,2,6} \right) v - {\det} \left( \mathbf{S}_{\ 1,2,3}^{\ 4,2,7} \right) w \Big] p_{{x}}
+ \Big[ {\det} \left( \mathbf{S}_{\ 1,2,3}^{\ 4,3,5} \right) u - {\det} \left( \mathbf{S}_{\ 1,2,3}^{\ 4,3,6} \right) v + 
\end{equation*}
\begin{equation*} 
{\det} \left( \mathbf{S}_{\ 1,2,3}^{\ 4,3,7} \right) w \Big] p_{{y}} \Big) w  + \Big( \Big[ {\det} \left( \mathbf{S}_{\ 1,2,3}^{\ 5,2,6} \right)v- {\det} \left( \mathbf{S}_{\ 1,2,3}^{\ 5,2,7} \right) w \Big] p_{{x}} + \Big[ -{\det} \left( \mathbf{S}_{\ 1,2,3}^{\ 5,3,6} \right) v 
\end{equation*}
\begin{equation*}
+ {\det} \left( \mathbf{S}_{\ 1,2,3}^{\ 5,3,7} \right) w \Big] p_{{y}} + \Big[ {\det} \left( \mathbf{S}_{\ 1,2,3}^{\ 5,4,6} \right) v - {\det} \left( \mathbf{S}_{\ 1,2,3}^{\ 5,4,7} \right) w \Big] p_{{z}} \Big) p_{{x}}
- \Big( \Big[ {\det} \left( \mathbf{S}_{\ 1,2,3}^{\ 6,2,5} \right) u  
\end{equation*}
\begin{equation*} 
- {\det} \left( \mathbf{S}_{\ 1,2,3}^{\ 6,2,7} \right) w \Big] p_{{x}}+ \Big[-{\det} \left( \mathbf{S}_{\ 1,2,3}^{\ 6,3,5} \right) u + {\det} \left( \mathbf{S}_{\ 1,2,3}^{\ 6,3,7} \right) w \Big] p_{{y}}
+ \Big[ {\det} \left( \mathbf{S}_{\ 1,2,3}^{\ 6,4,5} \right) u 
\end{equation*}
\begin{equation*} 
- {\det} \left( \mathbf{S}_{\ 1,2,3}^{\ 6,4,7} \right) w \Big] p_{{z}} \Big) p_{{y}} + \Big( \Big[ {\det} \left( \mathbf{S}_{\ 1,2,3}^{\ 7,2,5} \right) u - {\det} \left( \mathbf{S}_{\ 1,2,3}^{\ 7,2,6} \right) v \Big]p_{{x}}
+ \Big[ -{\det} \left( \mathbf{S}_{\ 1,2,3}^{\ 7,3,5} \right) u  
\end{equation*}
\begin{equation}
+ {\det} \left( \mathbf{S}_{\ 1,2,3}^{\ 7,3,6} \right) v \Big]p_{{y}} + \Big[ {\det} \left( \mathbf{S}_{\ 1,2,3}^{\ 7,4,5} \right) u - {\det} \left( \mathbf{S}_{\ 1,2,3}^{\ 7,4,6} \right) v \Big]  p_{{z}} \Big) p_{{z}}=0.
\label{undesired:6}
\end{equation}
Note that the above polynomial is again ${\det} \left( \mathbf{S}_{\ 1}^{\ 1} \right)$ of Eq.\ (\ref{BorrasMatrix}). 
Eq.\ (\ref{undesired:6}), for all pose variables, gives ${\det} \left( \mathbf{S}_{\ 1,2,3}^{\ 6,2,5} \right)={\det} \left( \mathbf{S}_{\ 1,2,3}^{\ 6,3,5} \right)={\det} \left( \mathbf{S}_{\ 1,2,3}^{\ 7,2,5} \right)={\det} \left( \mathbf{S}_{\ 1,2,3}^{\ 6,3,6} \right)=0$, which once again upon transforming into brackets will be as follows: 
\begin{equation}
[\mathrm{r}, \mathrm{Y}, \mathrm{Z}, \mathrm{rZ}]=[\mathrm{r}, \mathrm{X}, \mathrm{Z}, \mathrm{rZ}]=0,
\label{bracket:11}
\end{equation} 
\begin{equation}
[\mathrm{r}, \mathrm{Y}, \mathrm{Z}, \mathrm{rY}]=[\mathrm{r}, \mathrm{X}, \mathrm{Z}, \mathrm{rX}]=0.
\label{bracket:12}
\end{equation} 
Again by resorting to Lemma \ref{lemma:3}, Eq.\ (\ref{bracket:11}) and Eq.\ (\ref{bracket:12}) we obtain $ \mathbf{X}, \mathbf{rX},\mathbf{rY},\mathbf{rZ}\in\{\mathbf{r}, \mathbf{Y}, \mathbf{Z}\}$, which naturally leads to an \emph{architectural singularity}.
\end{proof}

\begin{figure}[t!] 
\begin{center}   
  \begin{overpic}[height=45mm]{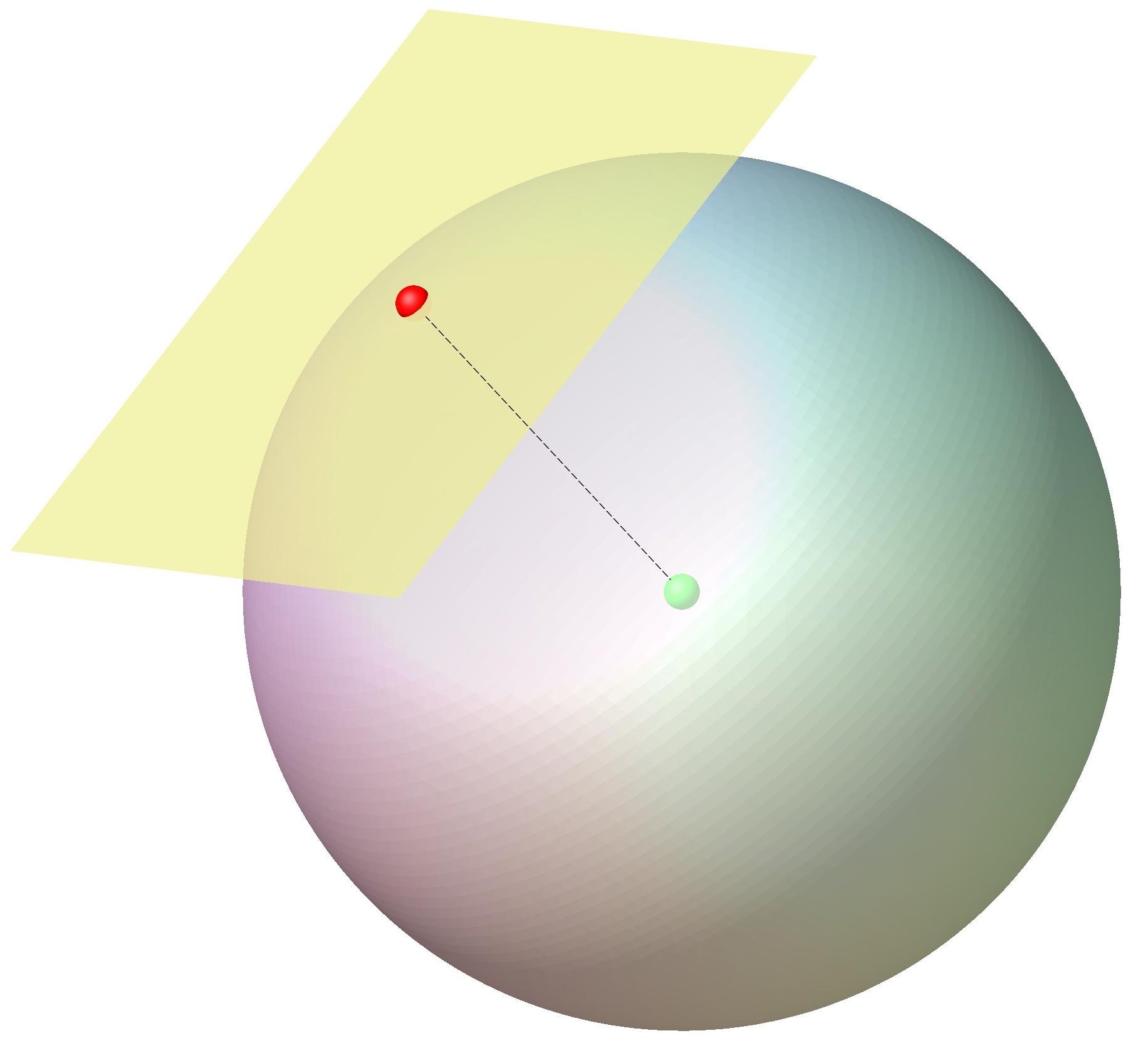}
  \end{overpic}
	\hfill
	\begin{overpic}[height=45mm]{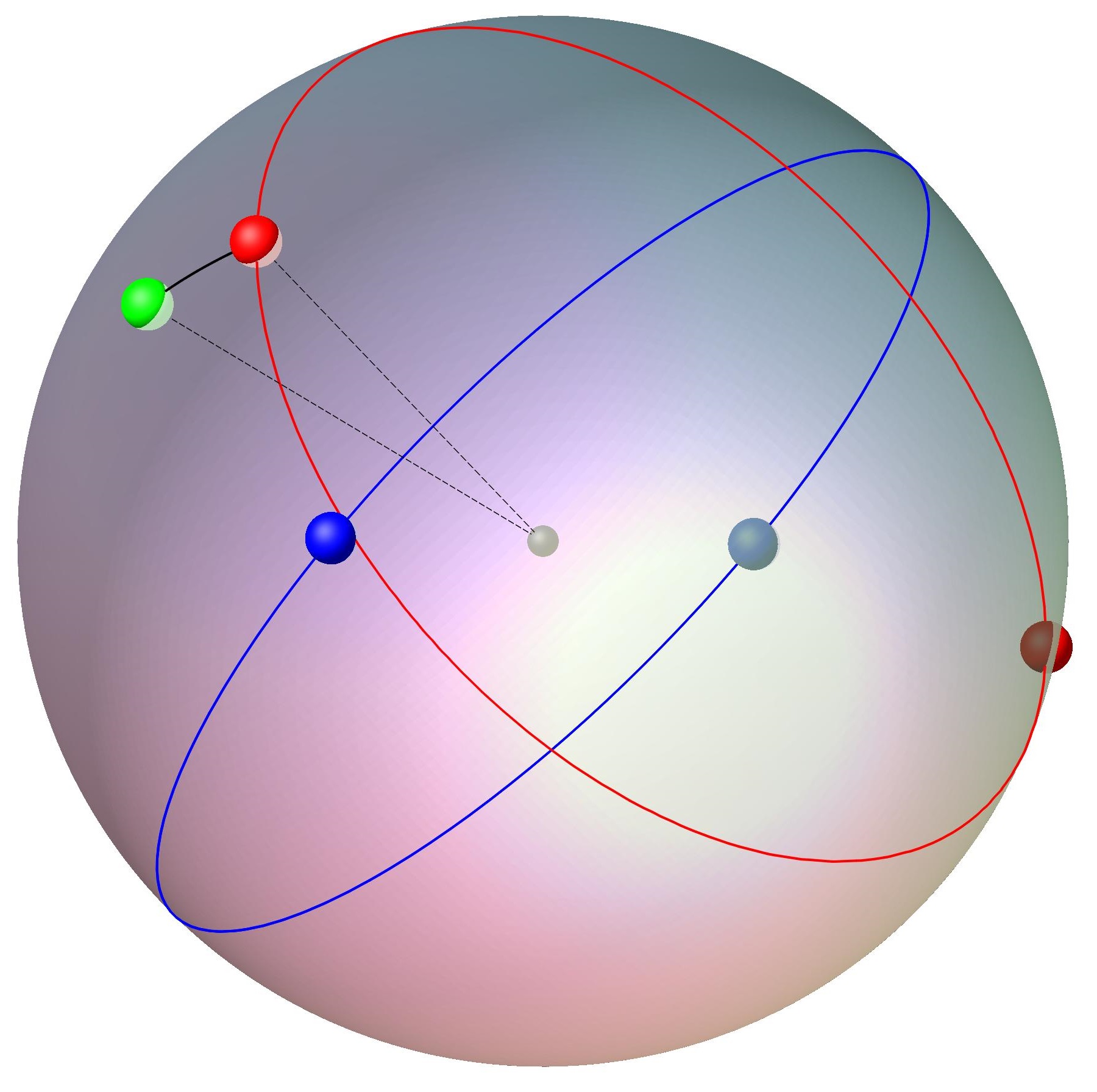}
  \end{overpic}
	\caption{
	Left: For fixed orientation the unique pedal point has coordinates  
	$(\frac{61}{33},\frac{38}{33},\frac{92}{33})\in\mathbb{R}^{3}$ which 
	has a distance of $1.21854359$ units from the given position. 
	Right: For fixed position the four pedal points are illustrated, where the one with coordinates 
	$(0.12661404,0.81506780,0.56536126)\in\mathbb{R}^{3}$ is closest to the 
	given orientation. The corresponding spherical distance equals $15.75049156^{\circ}$.}
	\label{fig:8}
\end{center}
\end{figure}


\section{Distance to singularity variety}
\label{sec:5}

In this section we compute \emph{singularity-free zones} for linear pentapods with a simple singularity variety studied in Section \ref{sec:2:1} and Section \ref{sec:3:1} respectively. 
We assume that the manipulator is always given in a non-singular pose $\mathfrak{G}=(g_1,\ldots ,g_6)\in\mathbb{R}^6$.

\subsection{\textbf{Linear in position variables}}
\label{sec:5:1}

The \emph{architecture matrix} of the linear pentapod used in the following examples is:
\begin{equation}
\mathbf{A}=\left( \begin {array}{ccccccc} 
1&-1/2&0&0&-1/2&0&0\\ 
2&1&2&0&2&4&0\\ 
4&-3&-1&0&-12&-4&0\\ 
6&-1&2&0&-6&12&0\end {array} \right), 
\label{Arch:num:1}
\end{equation}
where $\alpha=-2$ and $\beta=2$ in Eq.\ (\ref{res1}). Moreover we consider the non-singular pose $\mathfrak{G}=(\frac{1}{3},\frac{2}{3},\frac{2}{3},1,2,3)$.

\subsubsection{Fixed orientation case}
\label{sec:5:1:1}
We ask for the closest singular configuration $\mathfrak{O}$ having the same orientation $(g_1,g_2,g_3)$ as the given pose $\mathfrak{G}$. The distance to the singularity pose with 
respect to $(g_4,g_5,g_6)$ is computed according to the ordinary \emph{Euclidean} metric.
The \emph{singularity polynomial} is linear in position variables and under fixed orientation condition it will be a plane \emph{passing through the origin} in position space $\mathbb{R}^{3}$. Naturally, there will be only one pedal point (cf.\ Fig.\ \ref{fig:8}-left) and hence the number of solutions in this case will only be one.  
Moreover $\mathfrak{O}=(\frac{1}{3}, \frac{2}{3},\frac{2}{3},\frac{61}{33},\frac{38}{33},\frac{92}{33})$ is illustrated in Fig.\ \ref{fig:position}-left.

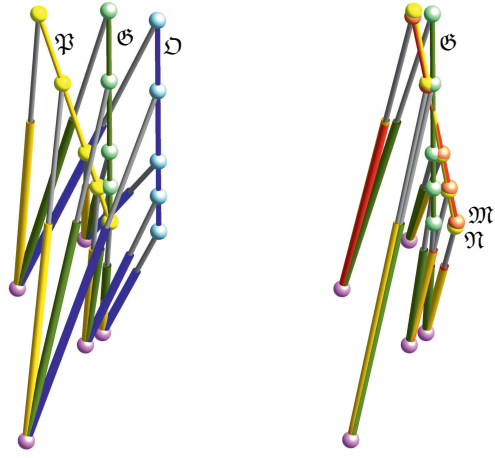
\begin{figure}[t!] 
\begin{center}   
  \begin{overpic}[height=60mm]{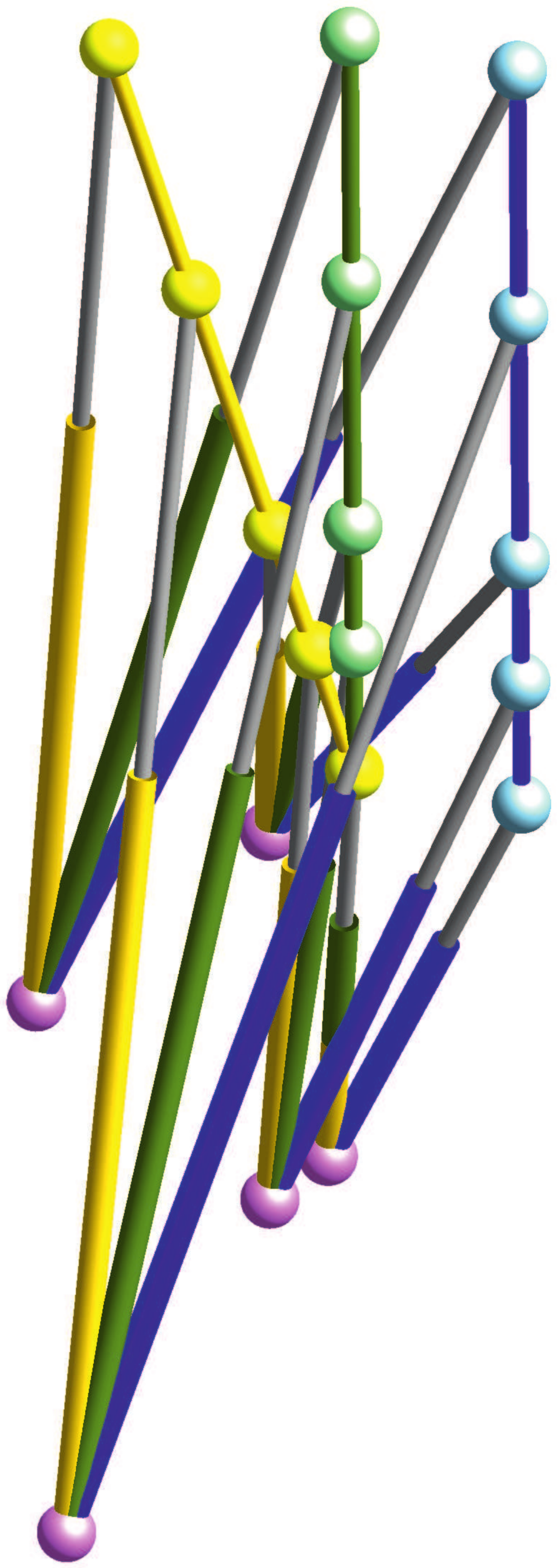}
	\begin{small}
	\put(10.5,89){$\mathfrak{P}$}
	\put(24,90){$\mathfrak{G}$}
	\put(34.5,88){$\mathfrak{O}$}
	\end{small}     
  \end{overpic} 
	\qquad\qquad\qquad
	 \begin{overpic}[height=60mm]{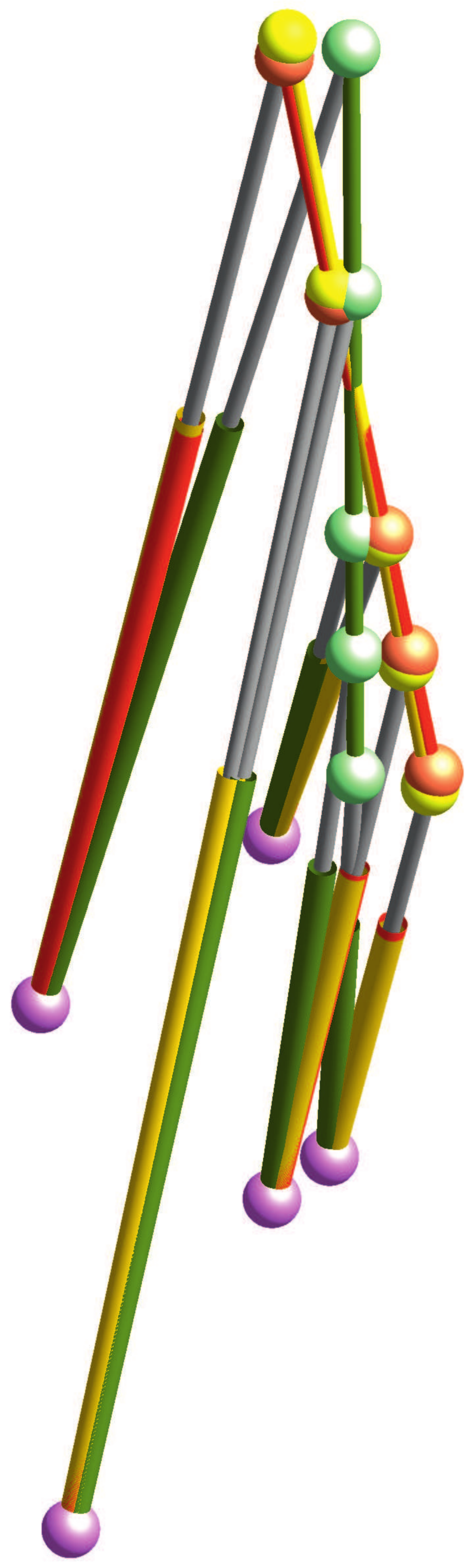}
	\begin{small}
	\put(29.5,45){$\mathfrak{N}$}
	\put(30.5,50.5){$\mathfrak{M}$}
	\put(24,89){$\mathfrak{G}$}
	\end{small}     
  \end{overpic} 
	\caption{
	Illustration of the pose $\mathfrak{G}$ (green) of the linear pentapod studied in Section \ref{sec:5:1}. 
	Left: The closest singular configurations in the position/orientation workspace 
	are given by the pose $\mathfrak{P}$ (yellow) and $\mathfrak{O}$ (blue), respectively. 
	Right: $\mathfrak{M}$ (red) is the closest singular pose under Euclidean motions of $\ell$ and
	$\mathfrak{N}$ (yellow) is the closest singularity under equiform motions  of $\ell$.}
	\label{fig:position}
\end{center}
\end{figure}

\subsubsection{Fixed position case}
\label{sec:5:1:2}
Now we ask for the closest singular configuration $\mathfrak{P}$, which has the same position $(g_4,g_5,g_6)$ as the given pose $\mathfrak{G}$. 
In this case the distance to the singularity curve with respect to $(g_1,g_2,g_3)$ is computed according to the \emph{Riemannian distance} $s$ on the sphere.

Under the fixed position the \emph{singularity polynomial} factors into two planes in $\mathbb{R}^{3}$:
\begin{equation}
w(A_{1} u + A_{2} v + A_{3} w + A_{4}) = 0, 
\label{two-circles}
\end{equation}
where the design variables are encoded in the coefficients $A_{i}$.

As a consequence the \emph{singular orientations} are obtained as the intersection of these two planes with the unit-sphere, which is given 
by the normalizing condition $\Gamma$. One of these planes always passes through the center of the sphere and hence the intersection is 
a great circle. For the second plane different cases can occur: 
\begin{enumerate}[(a)]
\item $A_{4}^{2}< A_{1}^{2}+A_{2}^{2}+A_{3}^{2}$: the plane intersects the sphere. 
\item $A_{4}^{2}= A_{1}^{2}+A_{2}^{2}+A_{3}^{2}$: the plane is tangent to the sphere. 
\item $A_{4}^{2}> A_{1}^{2}+A_{2}^{2}+A_{3}^{2}$: the plane doesn't intersect the sphere. 
\end{enumerate}
Depending on the case the total number of pedal points equals (a) $4$, (b) $3$ and (c) $2$, respectively.
In the example at hand there exist 4 pedal points, which are illustrated in Fig.\ \ref{fig:8}-right and are listed in Table \ref{table:1}.
Moreover $\mathfrak{P}= (0.1266,0.815,0.5653,1,2,3)$ is illustrated in Fig.\ \ref{fig:position}-left.

\begin{remark}
It should be noted that if the given non-singular orientation is normal to one of the planes intersecting the 
 unit-sphere, then there exists an infinite number of pedal points. \hfill $\diamond$
\end{remark}
 
\begin{table}[h!]
\begin{footnotesize}
\centering
\begin{tabular}{|c|c|c|c|c|}
\hline 
\phantom{1} & $u$ & $v$ & $w$ &  $s$  \\ \hline \hline 
1 &  0.12661404  & 0.81506780  & 0.56536126                   &  $15.75049156^{\circ}$   \\ \hline
2 &  0.44721359  & 0.89442719  & 0                             & $41.83152170^{\circ}$   \\ \hline
3 & -0.44721359 & -0.89442719 & 0                             & $138.25977700^{\circ}$   \\ \hline
4 & -0.60029825 & -0.34138359 & -0.72325600                  &  $155.56475890^{\circ}$   \\ \hline
\end{tabular}
\caption{The 4 real solutions in ascending order with respect to the spherical distance $s$ 
to the given orientation. 
}
\label{table:1}
\end{footnotesize}  
\end{table}

\subsubsection{General case}
\label{sec:5:1:3}
The general case deals with mixed (translational and rotational) DOFs, thus 
the question of a suitable distance function arises.  
As the configuration space $\mathcal{C}$ equals the space of oriented line-elements, 
we can adopt the object dependent metrics discussed in \cite{nawratil2017point} as follows:
\begin{equation}\label{distance}
d(\mathfrak{L}, \mathfrak{L^{'}})^{2}:=
\frac{1}{5}\sum_{j=1}^5{\|\mathbf{m}_j-\mathbf{m}^{'}_{j}\|}^{2},
\end{equation}
where $\mathfrak{L}$ and $\mathfrak{L^{'}}$ are two configurations and $\mathbf{m}_j$ and $\mathbf{m}^{'}_{j}$ denote
the coordinate vectors of the corresponding platform anchor points. 
This metric has already been used in \cite{rasoulzadeh2018rational} for the mechanical device at hand.

With respect to this metric $d$ we can compute the closest singular configuration $\mathfrak{M}$ to $\mathfrak{G}$ in the following way: 
We determine the set of pedal-points on the singularity variety with respect to $\mathfrak{G}$ as the variety 
$V(\tfrac{\partial L}{\partial u}, \tfrac{\partial L}{\partial v}, \tfrac{\partial L}{\partial w}, \tfrac{\partial L}{\partial p_{x}}, 
\tfrac{\partial L}{\partial p_{y}}, \tfrac{\partial L}{\partial p_{z}}, \tfrac{\partial L}{\partial \lambda_{1}}, 
\tfrac{\partial L}{\partial \lambda_{2}})$
where $\lambda_1$ and $\lambda_2$ are the Lagrange multipliers of the Lagrange equation: 
\begin{equation}
L(u,v,w,p_{x},p_{y}, p_{z}, \lambda_{1}, \lambda_{2}):=d(\mathfrak{M}, \mathfrak{G})^{2} + \lambda_{1}(u^2+v^2+w^2-1) + \lambda_{2}F.
\end{equation} 

Note that here $F$ is the singularity polynomial linear in position variables, obtained from Theorem \ref{theorem:1}.
Considering the example of the design parameters indicated in Eq.\ (\ref{Arch:num:1}), 
there are 10 solutions out of which 6 are real\footnote{It is unknown if examples with $10$ real solutions can exist.}.

After solving $\{ \tfrac{\partial L}{\partial p_{x}}, 
\tfrac{\partial L}{\partial p_{y}}, \tfrac{\partial L}{\partial p_{z}}\}$ for $\{ p_{x}, p_{y}, p_{z}\}$ and substituting the values obtained into the rest of the equations of the system, we can use the Gr\"obner basis method to solve the new system for the  remaining variables. Using the order $w>v>u>\lambda_{2}>\lambda_{1}$
one of the Gr\"obner basis generators solely depends on $\lambda_{1}$ while the rest depend on $\lambda_{1}$ and another orientation variable or $\lambda_2$, respectively. Based on this elimination technique the following table is obtained:  
 
\begin{table}[h!]
\begin{footnotesize}
\centering
\begin{tabular}{|c|c|c|c|c|c|c|}
\hline 
\phantom{1} & $u$ & $v$ & $w$ & $\lambda_1$ & $\lambda_2$ & $d$  \\ \hline \hline 

1 & 0.19954344     	&  0.75426388     &  0.62551450    	&  0.22471412    &   0.00242829      & 0.37163905   \\ \hline
2 & 0.44721359      &  0.89442721     &  0.00000000   	& -1.18154819    &   0.15475648      & 1.53723662   \\ \hline
3 & -0.44720571     & -0.89444123    	&  0.00001503    	& -8.09845180    &   0.00888318      & 4.02454431   \\ \hline
4 & -0.72878205     & -0.23306556    	& -0.64396839    	& -9.46430882    &   0.00812550      & 4.13597163   \\ \hline
5 & 0.50116745      &  0.86532314     &  0.00686193    	& -1.24444052    &  63.53263267      & 4.98948239   \\ \hline
6 & -0.44100968     & -0.89750916   	& -0.00554456    	& -8.10658006    & -11.11676392      & 6.20308215   \\ \hline
\end{tabular}
\caption{The 6 real solutions in ascending order with respect to the distance $d$ 
from $\mathfrak{G}$. The corresponding values of missing variables $p_x,p_y,p_z$ 
are obtained by substituting $u,v,w,\lambda_1,\lambda_2$  into the expressions for $p_x,p_y,p_z$. 
}
\label{table:2}
\end{footnotesize}
\end{table}
The first row in Table \ref{table:2} corresponds to the global minimizer $\mathfrak{M}$ 
illustrated in Fig.\ \ref{fig:position}-right, which has position variables
$p_x=1.42386285$, 
$p_y=1.69623807$ and 
$p_z=3.11364494$.


\subsubsection{General case without normalizing condition}
\label{sec:5:1:4}

We can simplify the problem by considering equiform transformations of the 
linear platform $\ell$, which is equivalent to the cancellation of the normalizing condition $\Gamma$.  
It turns out that for this reduced set of equations only $3$ pedal points exit over $\mathbb{C}$. 

For the example under consideration the computations can be done in the same way as in Section \ref{sec:5:1:4} 
with the sole difference that $\lambda_1$ is now absent. We end up with the following table:

\begin{table}[h!]
\begin{footnotesize}
\centering
\begin{tabular}{|c|c|c|c|c|c|c|}
\hline 
\phantom{1} & $u$ & $v$ & $w$ & $\lambda_2$ & $d$ & $\mu$  \\ \hline \hline 
1 & 0.22077150 &  0.77922849 &  0.65664594 &  0.00209764  & 0.35854952 	& 1.04265095 \\ \hline
2 & 0.33333333 &  0.66666666 &  0          &  0.04901408 	& 1.43604394  & 0.74535599 \\ \hline
3 & 0.36256185 &  0.63743814 &  0.01002046 & 26.26334956  & 4.95602764  & 0.73340227 \\ \hline
\end{tabular}
\caption{The 3 real solutions in ascending order with respect to the distance $d$ from $\mathfrak{G}$. 
The scaling factor of the corresponding 
equiform displacement of the platform is given by $\mu$.
The corresponding values of missing variables $p_x,p_y,p_z$ 
are obtained by substituting $u,v,w,\lambda_2$  into the expressions for $p_x,p_y,p_z$. 
}
\label{table:3}
\end{footnotesize}
\end{table}
The first row in Table \ref{table:3} corresponds to the global minimizer $\mathfrak{N}$ 
illustrated in Fig.\ \ref{fig:position}-right, which has position variables 
$p_x=1.36501824$, 
$p_y=1.63498176$ and 
$p_z=3.03249538$.

\begin{figure}[t!] 
\begin{center}   
  \begin{overpic}[height=45mm]{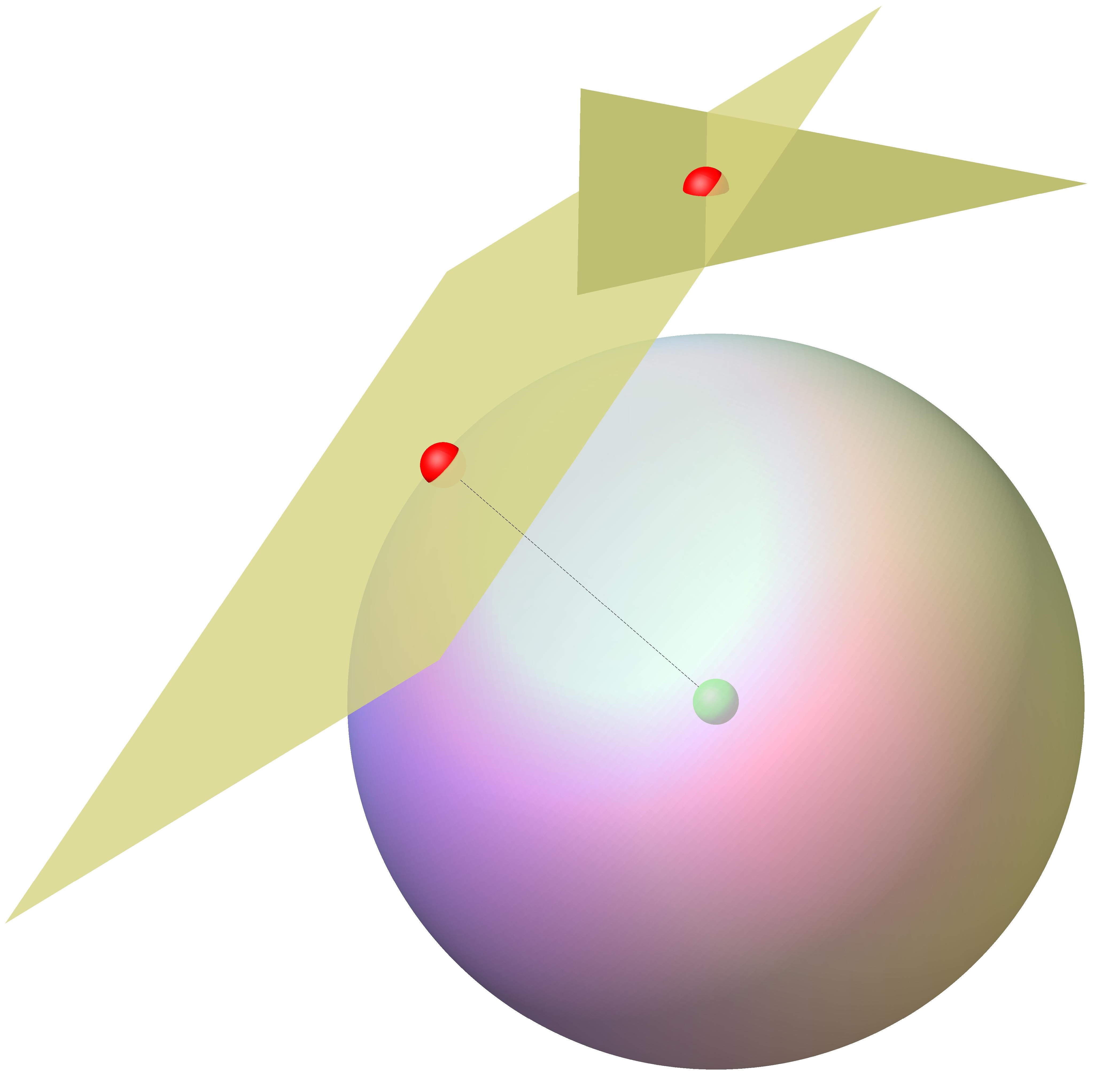}
  \end{overpic}
	\hfill
	\begin{overpic}[height=45mm]{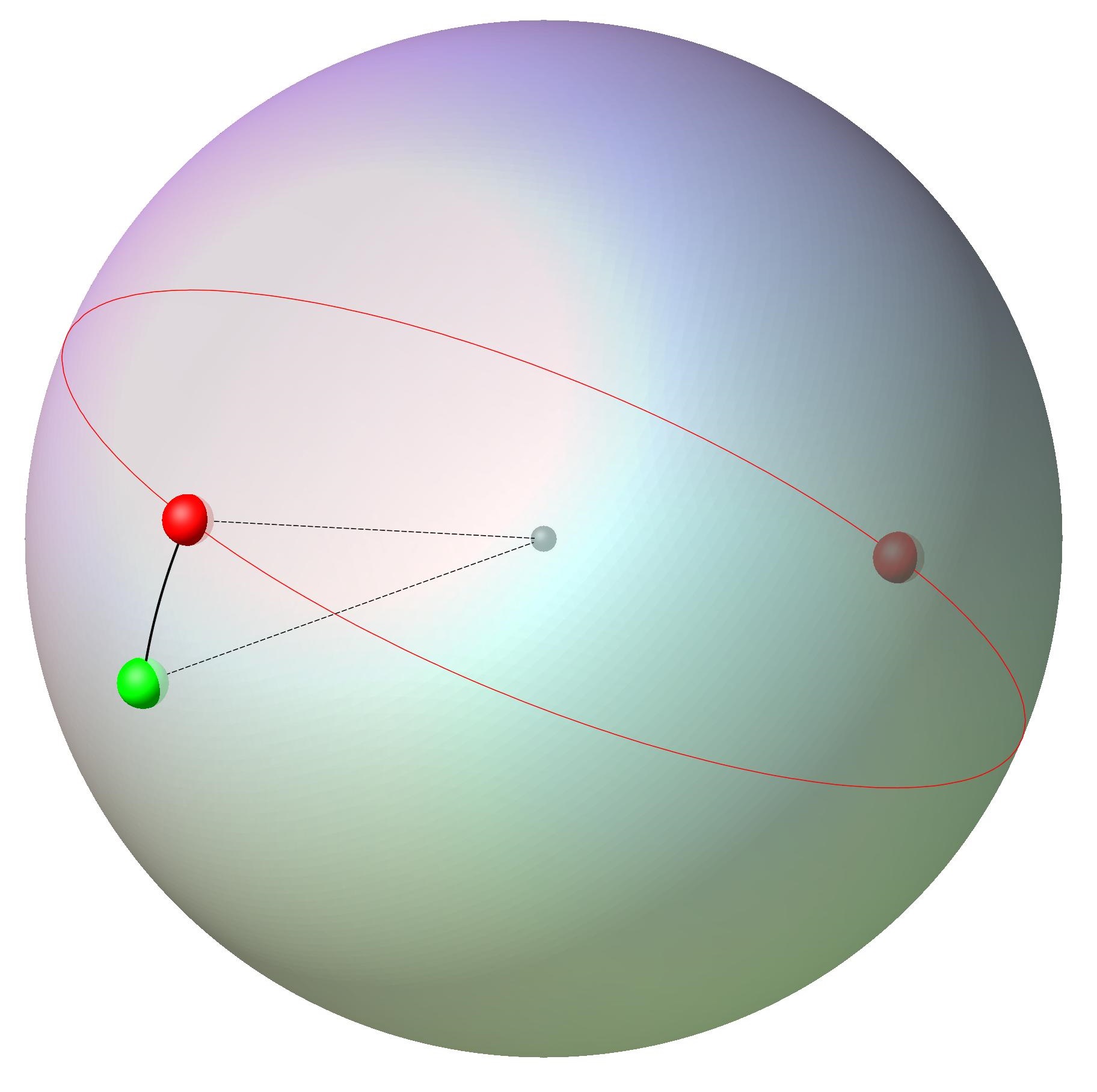}
  \end{overpic}
	\caption{
	Left: For fixed orientation there exits for each of the two 
	planes a unique pedal point. One pedal point has coordinates $(-8/17, 9/17, 12/17)$ and 
	a distance of $4.80196038$ units to the given position and the other pedal point has 
	coordinates $(2,3,0)$ and a distance of $4$ units.
	Right: For fixed position the two pedal points are illustrated, where the one with coordinate 
	$(0.11346545 , 0.47007115 , 0.87530491)$ is closest to the 
	given orientation. The corresponding spherical distance $s$ equals $20.82450533^{\circ}$. 
		The second pedal point is antipodal to the first one and the distance $s$ is the supplementary angle.}
	\label{fig:10}
\end{center}
\end{figure}

\subsection{\textbf{Linear in orientation variables}}
\label{sec:5:2}

The \emph{architecture matrix} of the linear pentapod used in the following examples is:
\begin{equation}
\mathbf{A}=\left( \begin {array}{ccccccc} 
1&1&0&0&1&0&0\\ 
3&-1/2&3/2&0&-3/2&9/2&0\\ 
5&-3&4&0&-15&20&0\\ 
6&-1&2&0&-6&12&0
\end {array} \right),
\label{Arch:num:2}
\end{equation}
where $\alpha=\beta=1$ in Eq.\ (\ref{equation:2}).  Moreover we consider the non-singular pose 
$\mathfrak{G}=(\frac{1}{3},\frac{2}{3},\frac{2}{3},1,2,3)$.

\subsubsection{Fixed orientation case}
\label{sec:5:2:1}
Once again we ask for the closest singular configuration $\mathfrak{O}$ having the same orientation $(g_1,g_2,g_3)$ as the given pose $\mathfrak{G}$. The distance to the singularity pose with respect to $(g_4,g_5,g_6)$ is computed according to the ordinary \emph{Euclidean} metric.
Under fixed orientation condition it is revealed that the \emph{singularity polynomial} is factored to:
\begin{equation}
p_z(B_{1} p_x + B_{2} p_y + B_{3} p_z + B_{4}) = 0,
\label{two-circles2}
\end{equation}
where again the design information is encoded in coefficients $B_i$. 
For each of the two planes in position space $\mathbb{R}^{3}$ we can compute the pedal point with respect 
to the given pose (cf.\ Fig.\ \ref{fig:10}-left). The closer pedal point implies 
$\mathfrak{O}=(\frac{1}{3}, \frac{2}{3},\frac{2}{3},2,3,0)$ illustrated in Fig.\ \ref{fig:orientation}-left.

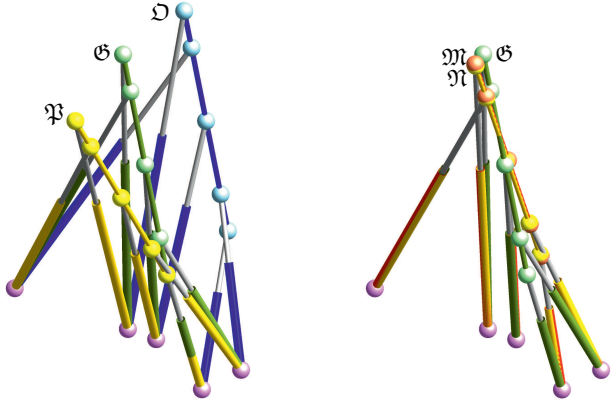
\begin{figure}[t!] 
\begin{center}   
  \begin{overpic}[width=33mm]{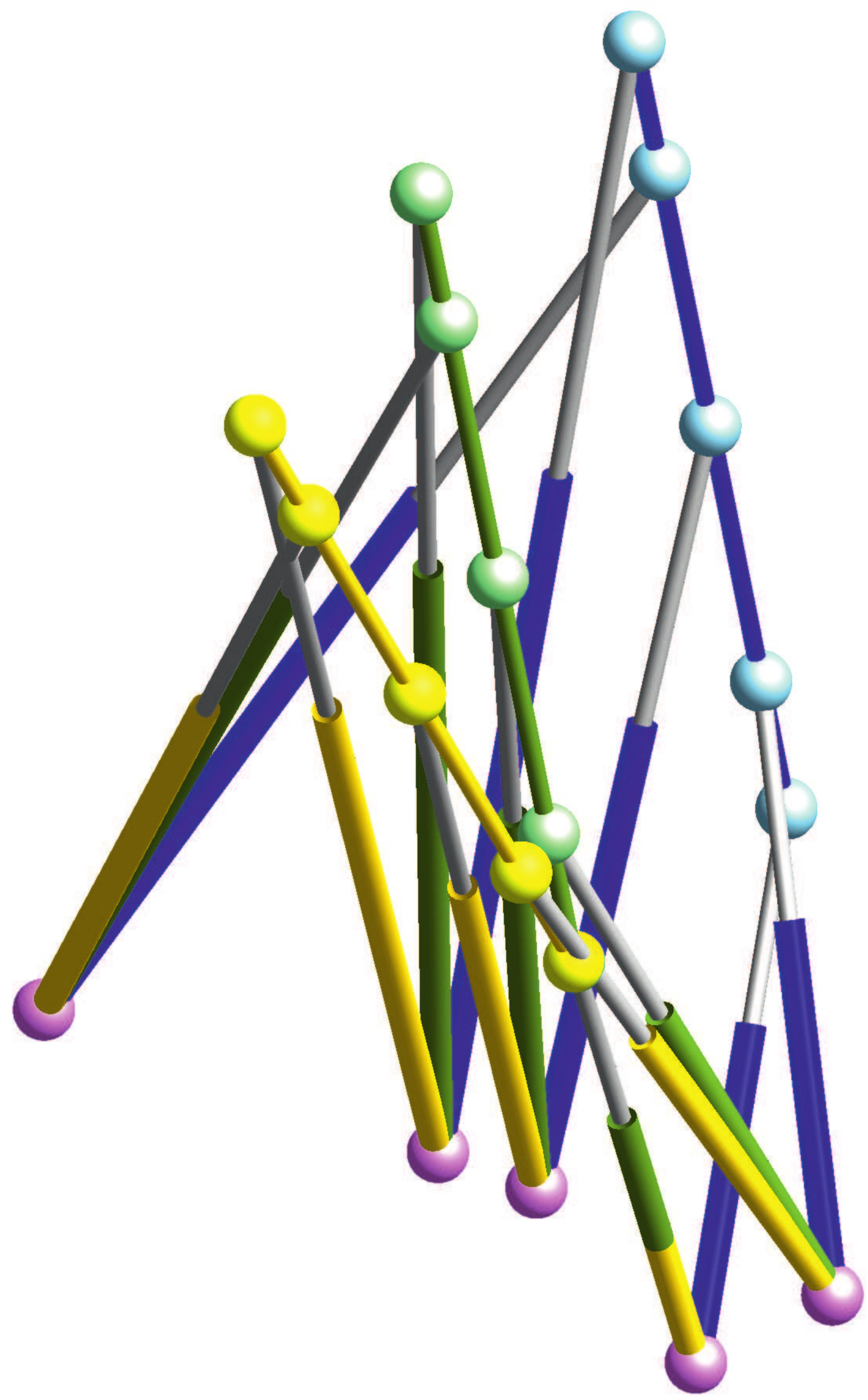}
	\begin{small}
	\put(10.5,70){$\mathfrak{P}$}
	\put(22.5,85){$\mathfrak{G}$}
	\put(37,95){$\mathfrak{O}$}
	\end{small}     
  \end{overpic} 
	\qquad\qquad
	 \begin{overpic}[width=33mm]{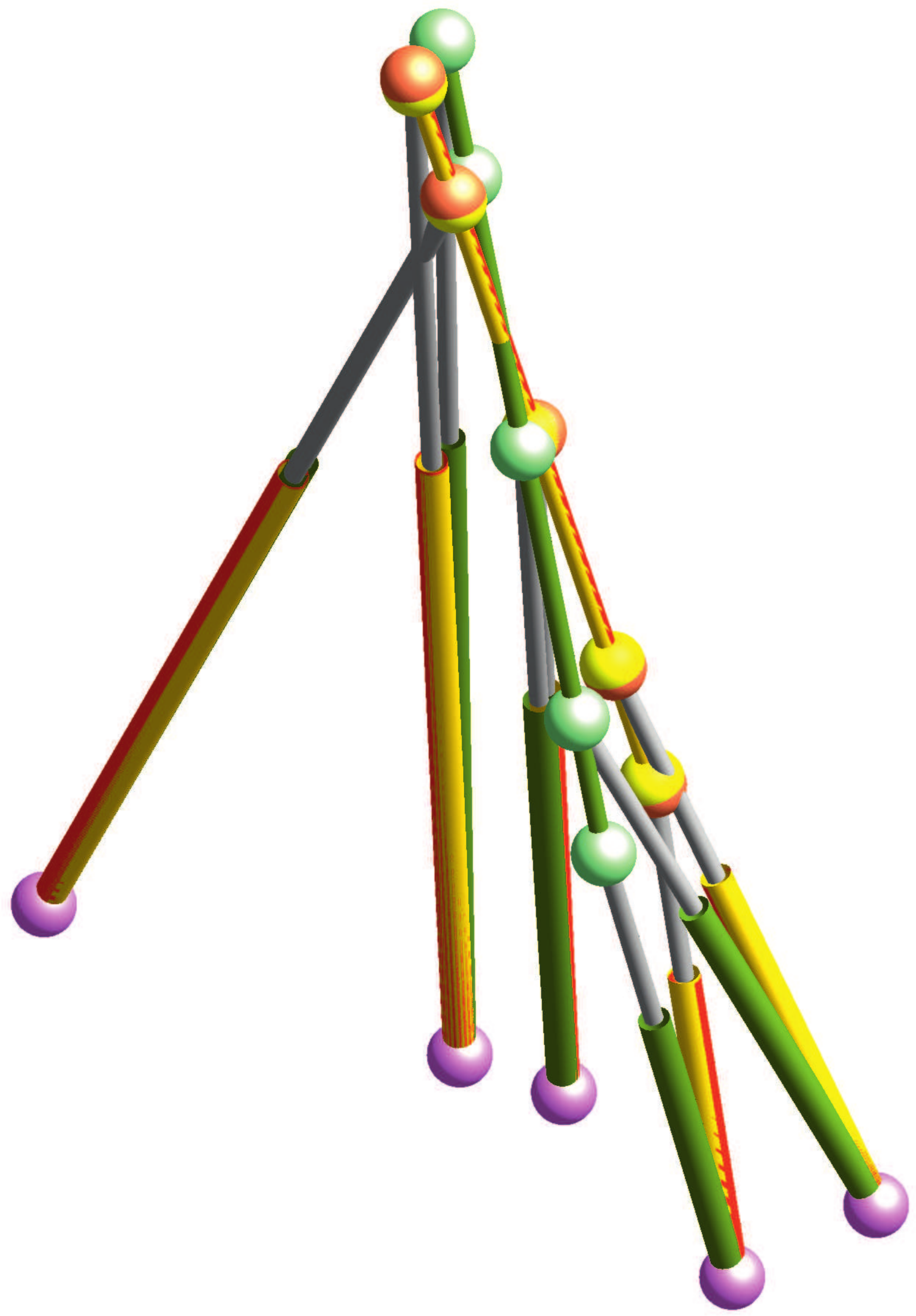}
	\begin{small}
	\put(23.5,88){$\mathfrak{N}$}
	\put(22,94){$\mathfrak{M}$}
	\put(37,95){$\mathfrak{G}$}
	\end{small}     
  \end{overpic} 
	\caption{
	Illustration of the pose $\mathfrak{G}$ (green) of the linear pentapod studied in Section \ref{sec:5:2}. 
	Left: The closest singular configurations in the position/orientation workspace 
	are given by the pose $\mathfrak{P}$ (yellow) and $\mathfrak{O}$ (blue), respectively. 
	Right: $\mathfrak{M}$ (red) is the closest singular pose under Euclidean motions of $\ell$ and
	$\mathfrak{N}$ (yellow) is the closest singularity under equiform motions  of $\ell$.}
	\label{fig:orientation}
\end{center}
\end{figure}

\subsubsection{Fixed position case}
\label{sec:5:2:2}
Now we ask again for the closest singular configuration $\mathfrak{P}$, 
which has the same position $(g_4,g_5,g_6)$ as the given pose $\mathfrak{G}$.
As the \emph{singularity polynomial} is linear in orientation variables and does not possess an absolute term,  
the \emph{singularity loci} is a \emph{great circle} for the fixed position case. If the given orientation differs from the 
pole of the great circle, then there exist two pedal points (otherwise infinitely many). 

The results for the example at hand are illustrated in Fig.\ \ref{fig:10}-right and 
the pose $\mathfrak{P}=(0.11346545 , 0.47007115 , 0.87530491,1,2,3)$ is displayed in Fig.\ \ref{fig:orientation}-left.

\subsubsection{General Case}
\label{sec:5:2:3}

Similar computations as in Section \ref{sec:5:1:3} show that there are again $10$ solutions out of which $6$ are real. 
They are given in the following table:

\begin{table}[h]
\begin{footnotesize}
\centering
\begin{tabular}{|c|c|c|c|c|c|c|}
\hline 
\phantom{1} & $u$ & $v$ & $w$ & $\lambda_1$ & $\lambda_2$ & $d$  \\ \hline \hline 
1 &  0.24002202   &  0.57831003     &  0.77970951     &  -0.07616071  	&  0.00198708     & 0.41484860   \\ \hline
2 &  0.16067752   &  0.32134537     &  0.93924532     &   5.58789193    & -0.03317073      & 2.44661840    \\ \hline
3 & -0.20843306   & -0.55064498     & -0.80863487     & -10.27182281    &  0.00088059      & 4.53615852    \\ \hline
4 & -0.35275218   &  0.88481355     & -0.34421986     &  -3.79940039    &  0.38394736      & 6.70384275    \\ \hline
5 & -0.02624291   & -0.92437183     &  0.38072309     &  -7.13002767    &  0.19314992   	  & 7.16835476     \\ \hline
6 & -0.06268654   & -0.12537309     & -0.99012725     & -32.85080126    &  0.07233642   	& 9.04867032    \\ \hline
\end{tabular}
\caption{The 6 real solutions in ascending order with respect to the distance $d$ from $\mathfrak{G}$. 
}
\label{table:values}
\end{footnotesize}
\end{table}
The first row in Table \ref{table:values} corresponds to the global minimizer $\mathfrak{M}$ 
illustrated in Fig.\ \ref{fig:orientation}-right, which has position variables
$p_x=1.35978906$, 
$p_y=2.34492506$ and 
$p_z=2.57706069$.

\subsubsection{General case without normalizing condition}
\label{sec:5:2:4}

Similar computations as in Section \ref{sec:5:1:4} show  again that the number of solution reduces to three. 
For the example at hand all three are real and read as follows: 

\begin{table}[h]
\begin{footnotesize}
\centering
\begin{tabular}{|c|c|c|c|c|c|c|}
\hline 
\phantom{1} & $u$ & $v$ & $w$ & $\lambda_2$ & $d$ & $\mu$  \\ \hline \hline 
1 & 0.23632218 & 0.56965551 &  0.76841946   & 0.00196374   & 0.41349741  & 0.98530404  \\ \hline
2 & 0.33333333 & 0.66666666 &  1.30046948   &-0.02111913   & 1.81542685  & 1.49892509   \\ \hline
3 &-0.06965551 & 0.26367781 & -0.10175277   & 3.26647730   & 6.49924087  & 0.29108677  \\ \hline
\end{tabular}
\caption{The 3 real solutions in ascending order with respect to the distance $d$ from $\mathfrak{G}$. 
}
\label{table:5}
\end{footnotesize}
\end{table}
The first row in Table \ref{table:5} corresponds to the global minimizer $\mathfrak{N}$ 
illustrated in Fig.\ \ref{fig:orientation}-right, which has position variables
$p_x=1.36986410$, 
$p_y=2.36986410$ and 
$p_z=2.61205791$.

\begin{remark}
Through the numerical examples one observes that the same number of pedal points is obtained for the other two geometries listed in Theorem \ref{thm:orientation} for the computation of the closest singular pose under Euclidean motions of $\ell$ and 
equiform motions of $\ell$, respectively. 
\hfill $\diamond$
\end{remark}


\section{Conclusions}
\label{conclusion}

In this paper we computed linear pentapods with a simplified singularity variety. In detail we 
determined all non-architecturally singular designs where the singularity polynomial is 
\begin{itemize}
\item[$\bullet$] linear in position variables (cf.\ Section \ref{sec:2}), 
\item[$\bullet$] linear in orientation variables (cf.\ Section \ref{sec:3}),
\item[$\bullet$] quadratic in total (cf.\ Section \ref{sec:4}).
\end{itemize}     
We demonstrated in Section \ref{sec:5} that these designs imply a degree reduction of 
the polynomials associated with the problem of determining singularity-free zones. 
Especially the closest singular configurations under equiform motions (cf.\ Sections \ref{sec:5:1:4} and \ref{sec:5:2:4}) 
are of interest, as they can be computed in closed form. Therefore their deeper study is 
dedicated to future research.

Finally we conclude the paper by presenting three kinematic redundant designs of linear pentapods with a simple singularity variety. 
The designs proposed in Sections \ref{des1} and \ref{des2} have two dofs of kinematically redundancy and 
the design given in Section \ref{des3} has even three kinematic redundant dofs.

\subsection{Design 1}\label{des1}
This design, displayed in Fig.\ \ref{fig:12}, is based on the idea to change the 
coefficient $\beta$ of the affine coupling $\kappa$ given in Eq.\ (\ref{affine}) by a reconfiguration of the 
base.  This can be achieved by a suitable sliding of the base points. 
The fibers of the singular affine transformation $\kappa$ from the base plane to the platform 
correspond to parallel lines in the base plane. It is well known (cf.\ Section 4.3 of \cite{borras2011architectural}) that 
a reconfiguration of a base point along its corresponding fiber does not change the singularity variety. 
Therefore it suggests itself to mount the sliders orthogonal to the fiber-direction. 
This sliding gives the first degree of kinematical redundancy. 

\begin{remark}
The linear pentapod given in Fig.\ \ref{fig:12} has been designed in a symmetric way, such that 
the sliders of $M_i$ and $M_{i+1}$ (for $i=2,4$) have to move with the same velocity (but in opposite directions). 
Note that one can drive all sliders of $M_2,\ldots ,M_5$ with only one motor and a fixed gearing, as the ratio of 
the velocities of the sliders of $M_2$ and $M_4$ is constant. 

Moreover it can easily be checked, that the symmetric design proposed in Fig.\ \ref{fig:12}, can never be architecturally 
singular in practice. \hfill $\diamond$
\end{remark}

The second degree of kinematic redundancy is achieved by the sliding of the first base point 
in fiber-direction. 
As already mentioned this will not affect the singularity surface, but it can be used 
to increase the performance of the manipulator during an end-effector motion.

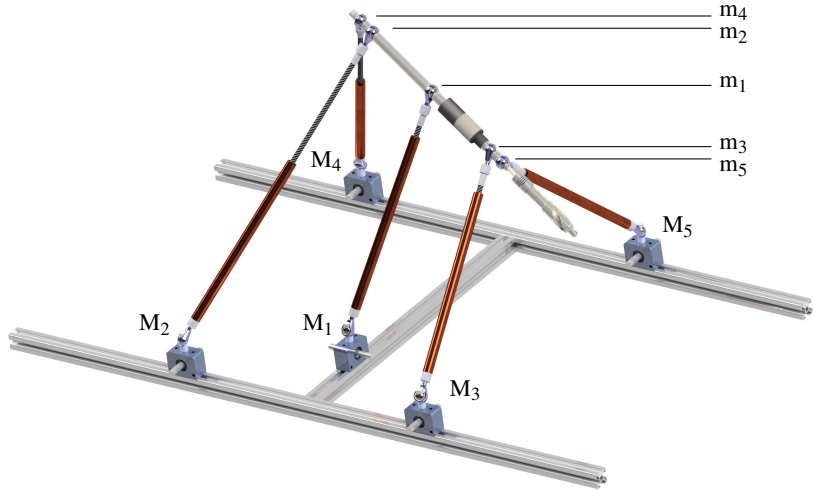
\begin{figure}[t!] 
\begin{center}   
  \begin{overpic}[height=65mm]{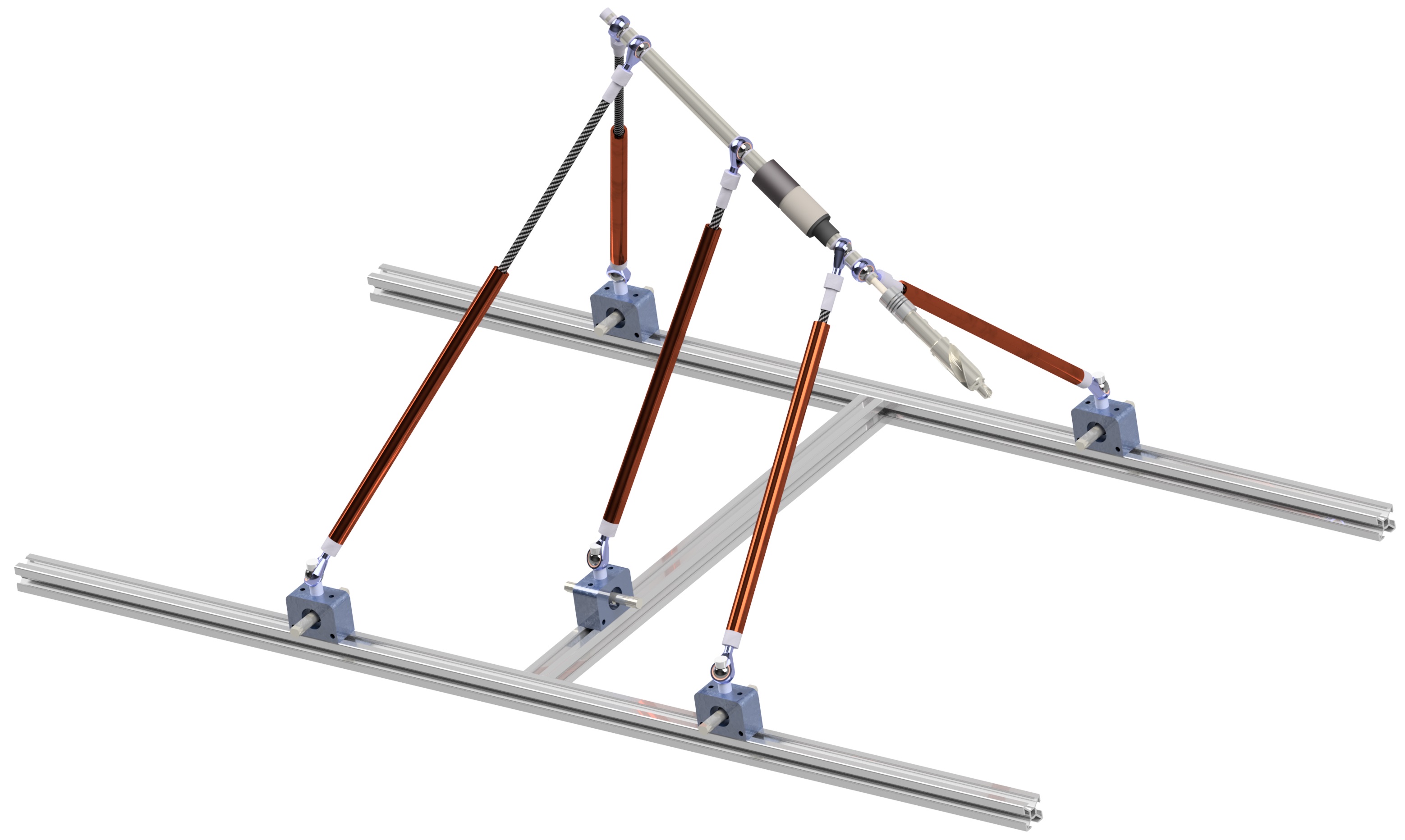}
	\begin{small}
   \put(46,58.5){\color{black}\line(1,0){41}}
   \put(48,57){\color{black}\line(1,0){39}}
   \put(54,50){\color{black}\line(1,0){33}}
   \put(61,42.5){\color{black}\line(1,0){26}}
   \put(63,41){\color{black}\line(1,0){24}}   
   \put(88,58.5){$\mathrm{m_{4}}$}
   \put(88,55.75){$\mathrm{m_{2}}$}
   \put(88,49.75){$\mathrm{m_{1}}$}
   \put(88,42.25){$\mathrm{m_{3}}$}
   \put(88,39.5){$\mathrm{m_{5}}$}
   \put(37,20){$\mathrm{M_{1}}$}
   \put(17,20){$\mathrm{M_{2}}$}
   \put(55,12){$\mathrm{M_{3}}$}
   \put(38,40){$\mathrm{M_{4}}$}
   \put(81,32){$\mathrm{M_{5}}$}
	\end{small}     
  \end{overpic}
	\caption{Illustration of the kinematic redundant linear pentapod of Section \ref{des1} with a linear singularity variety in position variables.}
	\label{fig:12}
\end{center}
\end{figure} 

\subsection{Design 2}\label{des2}
This design, based on item 1 of Theorem \ref{thm:orientation} and displayed in Fig.\ \ref{fig:13}, is also a 2-dof kinematically redundant pentapod with planar base, 
which has the property that its singular polynomial is linear in orientation for all possible configurations. 
The base points $M_2,\ldots ,M_5$ are collinearly mounted on a rod $g$, which slides (active joint) along a circular rail on the ground and is 
connected over a U-joint (passive joint) with the ceiling. Therefore the rod $g$ generates during the motion a right circular cone.   
 
For a better understanding of the redundant dofs, we have a look at the singular-invariant replacement of legs keeping the given platform anchor points:
\begin{enumerate}[$\star$]
\item
As this linear pentapod contains a line-line component (cf.\cite{borras2010singularity1}), 
one can relocate the base anchor points of the legs $m_2M_2, \ldots , m_5M_5$ arbitrarily on $g$ 
(assumed that the resulting manipulator is not architecturally singular). 
\begin{remark}\label{rem:redundant}
One can additionally allow a sliding (by active joints) of the base points along the rod $g$ (yielding further degrees of 
kinematical redundancy)  
but this will not change the singularity variety. These reconfigurations can only be used to 
improve the performance of the manipulator. \hfill $\diamond$
\end{remark}
\item
The base point of the first leg can be replaced by any point of the
plane spanned by $M_1$ and $g$ (assumed that the resulting manipulator is not architecturally singular). 
Therefore a sliding of $M_1$ along the circular rail changes the singularity variety. 
\end{enumerate}

\begin{figure}[t!] 
\begin{center}   
  \begin{overpic}[height=65mm]{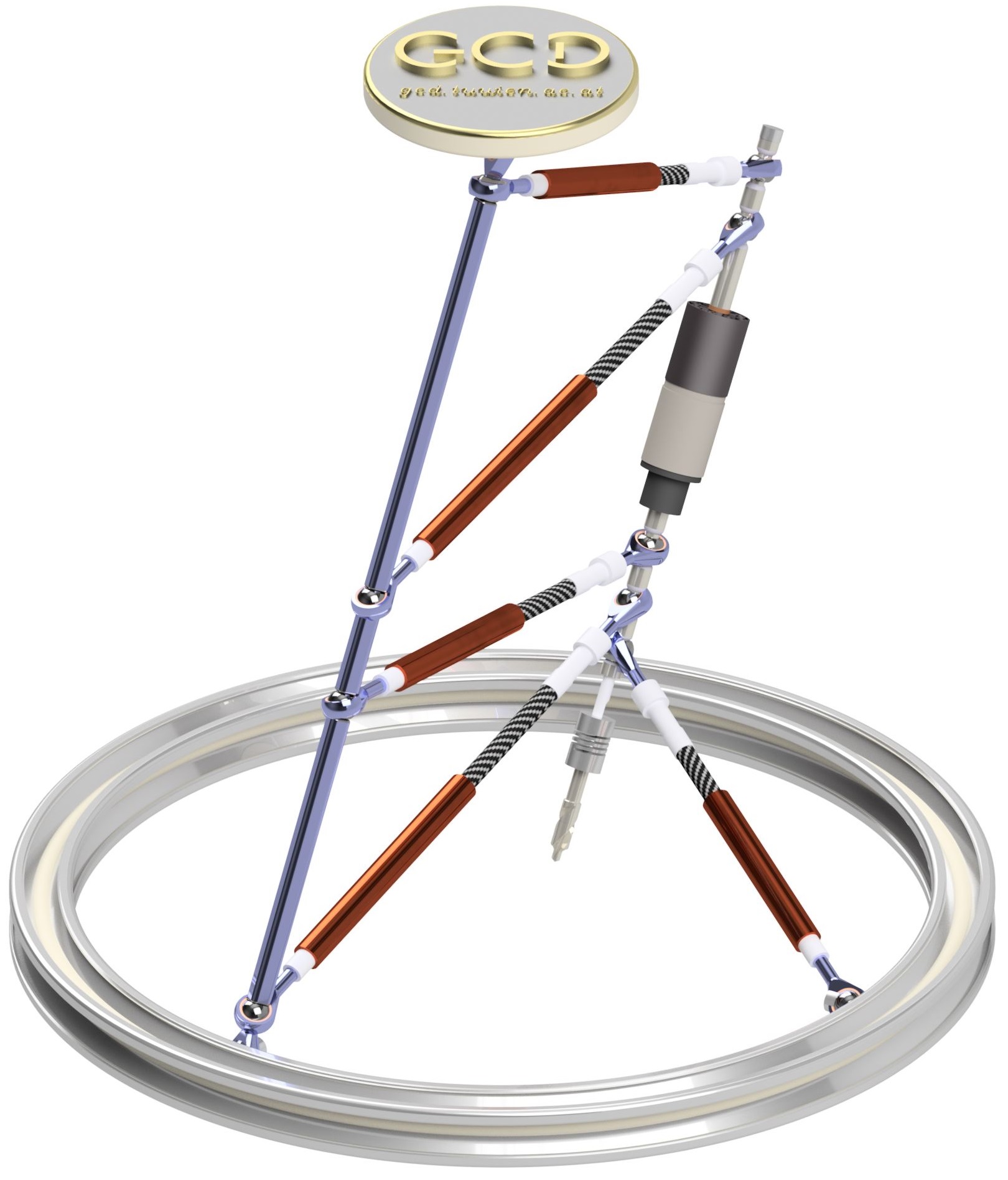}
	\begin{small}
   \put(-9,70){$g$}
   \put(89,43){$\mathrm{m_{1}}$}
   \put(89,48){$\mathrm{m_{2}}$}
   \put(89,53){$\mathrm{m_{3}}$}
   \put(89,80){$\mathrm{m_{4}}$}
   \put(89,85){$\mathrm{m_{5}}$}
   \put(93,14.5){$\mathrm{M_{1}}$}
   \put(-9,14){$\mathrm{M_{2}}$}
   \put(-9,40.5){$\mathrm{M_{3}}$}
   \put(-9,49){$\mathrm{M_{4}}$}
   \put(-9,84){$\mathrm{M_{5}}$}
   \put(56,45){\color{black}\line(1,0){30}}
   \put(58,49){\color{black}\line(1,0){28}}
   \put(59,53){\color{black}\line(1,0){27}}
   \put(67,81){\color{black}\line(1,0){19}}
   \put(69,85){\color{black}\line(1,0){17}}
   \put(75,15){\color{black}\line(1,0){17}}
   \put(-1,14.25){\color{black}\line(1,0){18}}
   \put(2,40.5){\color{black}\line(1,0){23}}
   \put(2,49){\color{black}\line(1,0){25}}
   \put(2,84){\color{black}\line(1,0){35}}
   \put(2,70){\color{black}\line(1,0){31}}	
	\end{small}     
  \end{overpic}
	\caption{Illustration of the kinematic redundant linear pentapod of Section \ref{des2} with a linear singularity variety in orientation variables. 
	The suggested design, where the upper part is mounted on the ceiling, can be of interest for e.g.\ the milling of an object without any need of its
	repositioning, as the manipulator can go around the object by $360$ degrees. }
	\label{fig:13}
\end{center}
\end{figure}
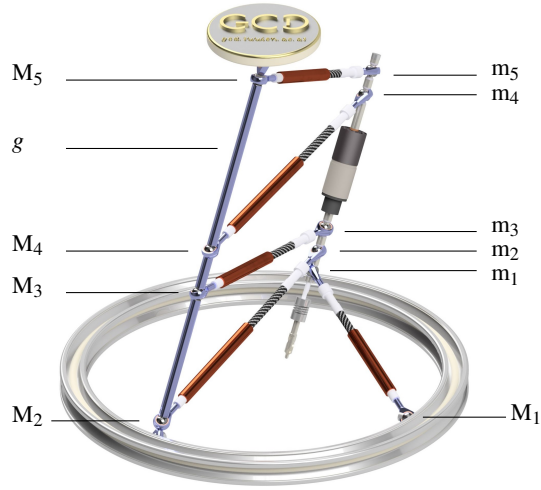 

\subsection{Design 3}\label{des3}
This design, based on item 2 of Theorem \ref{thm:orientation} and displayed in Fig.\ \ref{fig:14}, is a 3-dof kinematically redundant pentapod with planar base, 
which has the property that its singular polynomial is linear in orientation for all possible configurations. 
The anchor points $M_1$ and $M_2$ can slide along a circular rail (two active joints). The third degree of kinematic redundancy is obtained by the rotation of the 
rod $g$ on which the collinear points $M_3,M_4,M_5$ are mounted. 

For a better understanding of the redundant dofs, we study again the singular-invariant replacements of legs keeping the given platform anchor points:
\begin{enumerate}[$\star$]
\item
One can relocate the base anchor points of the legs $m_3M_3, m_4M_4 , m_5M_5$ arbitrarily on $g$ 
(assumed that the resulting manipulator is not architecturally singular). Therefore 
also Remark \ref{rem:redundant} holds in this context. 
\item
The base points of the first and second leg can be replaced by any two points of the
carrier plane of the circular rail (assumed that the resulting manipulator is not architecturally singular). 
As a consequence the sliding of $M_1$ and $M_2$ along the circular rail does not change the singularity variety. 
Therefore these two redundant dofs can only be used to improve the performance of the manipulator. 
\end{enumerate}

\begin{remark}
Finally it should be noted that a design, based on item 3 of Theorem \ref{thm:orientation}, is not suited for 
 technical realization due to the triple joint at the platform. \hfill $\diamond$
\end{remark}

\begin{figure}[t!] 
\begin{center}   
	\begin{overpic}[height=65mm]{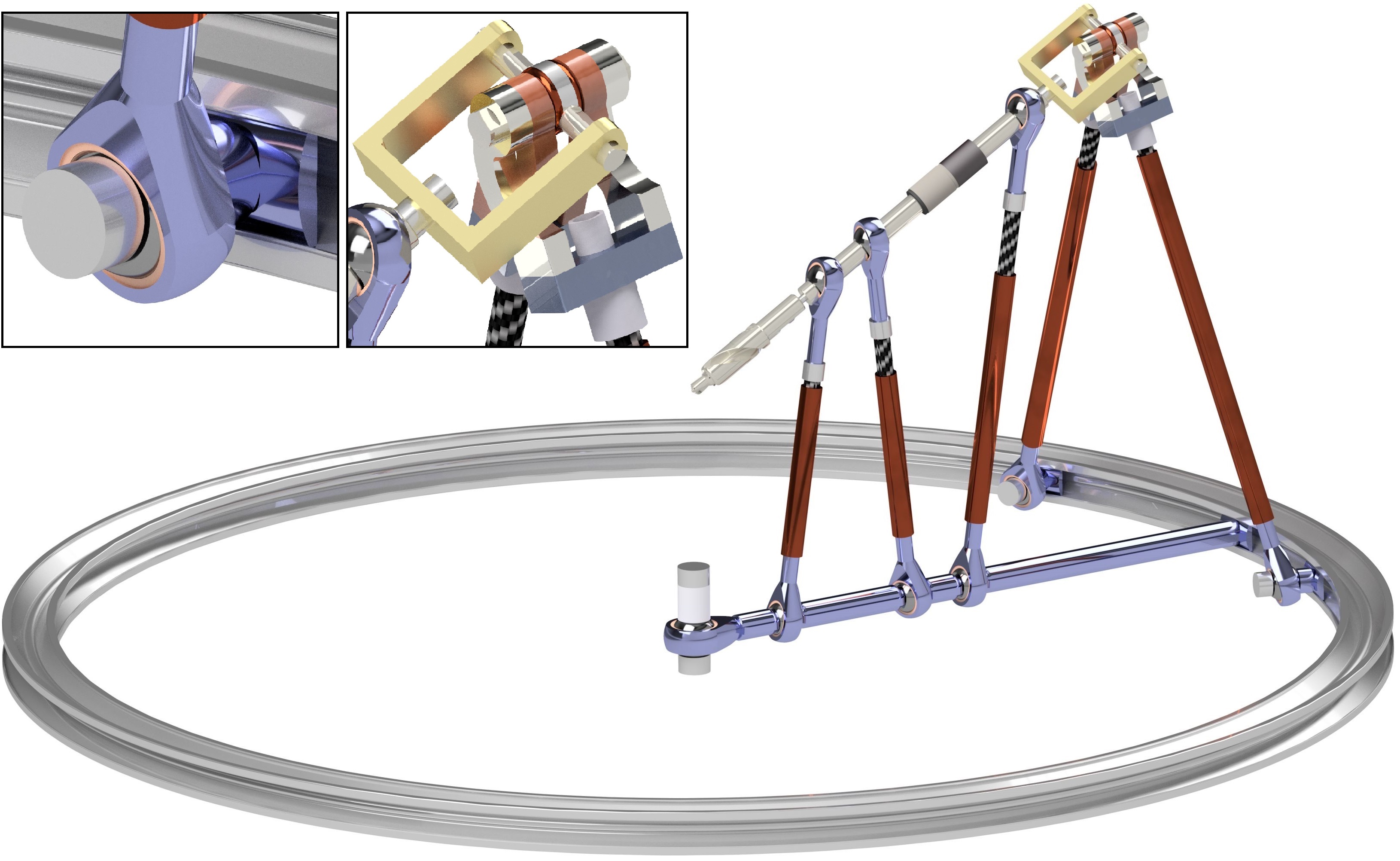}
	\begin{small}
   \put(58.7,44){\color{black}\line(0,1){4.5}}
   \put(62,47){\color{black}\line(0,1){5.5}}
   \put(62,52.5){\color{black}\line(-1,0){0.5}}
   \put(70,55){\color{black}\line(-1,0){8}}
   \put(73,60){\color{black}\line(-1,0){8}}
   \put(81.5,20.5){\color{black}\line(0,-1){5.5}}
   \put(57,49.5){$\mathrm{m_{5}}$}
   \put(57,52){$\mathrm{m_{4}}$}
   \put(57,54.5){$\mathrm{m_{3}}$}
   \put(53,59.25){$\mathrm{m_{1}=m_{2}}$}
   \put(54.5,12){$\mathrm{M_{5}}$}
   \put(63,12){$\mathrm{M_{4}}$}
   \put(68,12){$\mathrm{M_{3}}$}
   \put(73,12){$\mathrm{M_{2}}$}
   \put(87,15){$\mathrm{M_{1}}$}
   \put(81,12.5){$g$}
   \put(74,15.5){\color{black}\line(0,1){9}}
   \end{small}     
   \end{overpic}
	\caption{Illustration of the kinematic redundant linear pentapod of Section \ref{des3} with a linear singularity variety in orientation variables. 
	This design also allows a milling by $360$ degrees around the object. Moreover, detailed views of the double joint $m_1=m_2$ at the platform and the circular 
	slider of $M_2$ are provided. }
	\label{fig:14}
\end{center}
\end{figure}
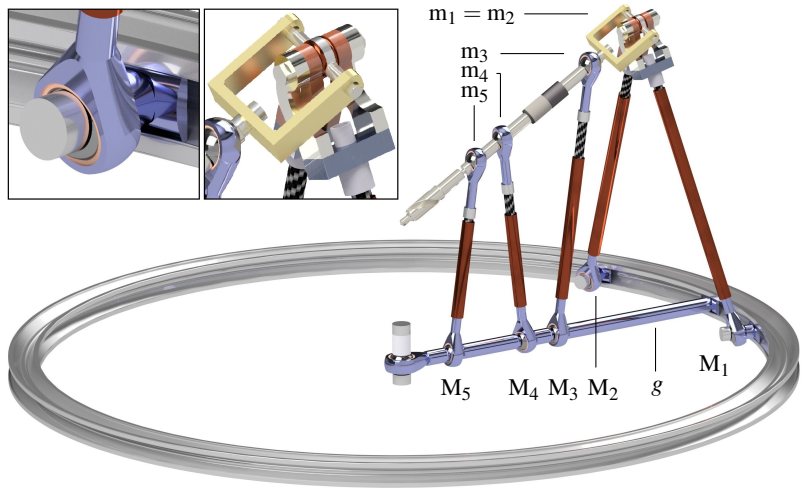 

\section*{Acknowledgement}
The research is supported by Grant No.~P~24927-N25 of the Austrian Science Fund FWF. 
Moreover the first author is funded by the Doctoral College "Computational Design" of TU Vienna.


\bibliography{mybibfile}

\begin{thebibliography}{10}
\expandafter\ifx\csname url\endcsname\relax
  \def\url#1{\texttt{#1}}\fi
\expandafter\ifx\csname urlprefix\endcsname\relax\def\urlprefix{URL }\fi
\expandafter\ifx\csname href\endcsname\relax
  \def\href#1#2{#2} \def\path#1{#1}\fi

\bibitem{kong2001generation}
X.~Kong, C.~M. Gosselin, Generation and forward displacement analysis of two
  new classes of analytic 6-{S}{P}{S} parallel manipulators, Journal of Field
  Robotics 18~(6) (2001) 295--304.

\bibitem{borras2010singularity}
J.~Borr{\`a}s, F.~Thomas, Singularity-invariant leg substitutions in pentapods,
  in: Intelligent Robots and Systems (IROS), 2010 IEEE/RSJ International
  Conference on, IEEE, 2010, pp. 2766--2771.

\bibitem{weck2002parallel}
M.~Weck, D.~Staimer, Parallel kinematic machine tools--current state and future
  potentials, CIRP Annals-Manufacturing Technology 51~(2) (2002) 671--683.

\bibitem{merlet1989singular}
J.-P. Merlet, Singular configurations of parallel manipulators and {G}rassmann
  geometry, The International Journal of Robotics Research 8~(5) (1989) 45--56.

\bibitem{pottmann2009computational}
H.~Pottmann, J.~Wallner, Computational line geometry, Springer Science \&
  Business Media, 2009.

\bibitem{rasoulzadeh2018rational}
A.~Rasoulzadeh, G.~Nawratil, Rational parametrization of linear pentapod’s
  singularity variety and the distance to it, in: Computational Kinematics,
  Springer, 2018, pp. 516--524.

\bibitem{zhang1992forward}
C.-D. Zhang, S.-M. Song, Forward kinematics of a class of parallel ({S}tewart)
  platforms with closed-form solutions, Journal of Field Robotics 9~(1) (1992)
  93--112.

\bibitem{nawratil2015self}
G.~Nawratil, J.~Schicho, Self-motions of pentapods with linear platform,
  Robotica (2015) 1--29.

\bibitem{nawratil2018line}
G.~Nawratil, On the line-symmetry of self-motions of linear pentapods, in:
  Advances in Robot Kinematics 2016, Springer, 2018, pp. 149--159.

\bibitem{borras2011singularity}
J.~Borr{\`a}s, F.~Thomas, C.~Torras, Singularity-invariant families of
  line-plane 5-$\mathrm{{S}}\mathrm{{P}}\mathrm{{U}}$ platforms, IEEE
  Transactions on Robotics 27~(5) (2011) 837--848.

\bibitem{karger2006stewart}
A.~Karger, {S}tewart-{G}ough platforms with simple singularity surface, in:
  Advances in Robot Kinematics, Springer, 2006, pp. 247--254.

\bibitem{nawratil2010stewart}
G.~Nawratil, {S}tewart {G}ough platforms with non-cubic singularity surface,
  Mechanism and Machine Theory 45~(12) (2010) 1851--1863.

\bibitem{nawratil2010}
G.~Nawratil, {S}tewart {G}ough platforms with linear singularity surface, in:
  Robotics in Alpe-Adria-Danube Region (RAAD), 2010 IEEE 19th International
  Workshop on, IEEE, 2010, pp. 231--235.

\bibitem{arvin}
A.~Rasoulzadeh, G.~Nawratil, Rational parametrization of linear pentapod’s
  singularity variety and the distance to it, Extended version on
  arXiv:1701.09107(2017).

\bibitem{nawratil2009newarch}
G.~Nawratil, A new approach to the classification of architecturally singular
  parallel manipulators, Computational Kinematics (2009) 349--358.

\bibitem{ben2008singulab}
P.~Ben-Horin, M.~Shoham, S.~Caro, D.~Chablat, P.~Wenger, Singulab--a graphical
  user interface for the singularity analysis of parallel robots based on
  {G}rassmann-{C}ayley algebra, Advances in Robot Kinematics: Analysis and
  Design (2008) 49--58.

\bibitem{white1994grassmann}
N.~L. White, {G}rassmann-{C}ayley algebra and robotics, Journal of Intelligent
  $\&$ Robotic Systems 11~(1) (1994) 91--107.

\bibitem{borras2011architectural}
J.~Borr{\`a}s, F.~Thomas, C.~Torras, Architectural singularities of a class of
  pentapods, Mechanism and Machine theory 46~(8) (2011) 1107--1120.

\bibitem{nawratil2017point}
G.~Nawratil, Point-models for the set of oriented line-elements--a survey,
  Mechanism and Machine Theory 111 (2017) 118--134.

\bibitem{borras2010singularity1}
J.~Borras, F.~Thomas, C.~Torras, Singularity-invariant leg rearrangements in
  {S}tewart--{G}ough platforms, Advances in Robot Kinematics: Motion in Man and
  Machine (2010) 421--428.

\end{thebibliography}

\end{document}